\documentclass{article}
\usepackage{graphicx}
\usepackage{subcaption}
\usepackage{booktabs} 
\usepackage{makecell} 
\usepackage{hyperref}
\usepackage{comment}
\usepackage{amsmath}
\usepackage{amssymb}
\usepackage{mathtools}
\usepackage{amsthm}
\usepackage{needspace}
\usepackage{bbm}
\usepackage[capitalize,noabbrev]{cleveref}

\theoremstyle{plain}

\newtheorem{thm}{Theorem}[section]
\newtheorem{prop}[thm]{Proposition}

\newtheorem{rem}[thm]{Remark}
\newtheorem{cor}[thm]{Corollary}

\newtheorem{defn}[thm]{Definition}

\newcommand{\Z}{\mathbb{Z}}
\newcommand{\R}{\mathbb{R}}

\newcommand{\cT}{\mathcal{T}}
\newcommand{\cX}{\mathcal{X}}
\newcommand{\cU}{\mathcal{U}}

\newcommand{\Dom}{\mathrm{Dom}}

\newcommand{\IMM}{\mathrm{IMM}}

\newcommand{\cC}{\mathcal{C}}
\newcommand{\cD}{\mathcal{D}}

\newcommand{\cL}{\mathcal{L}}

\DeclareMathOperator*{\argmax}{arg\,max}

\usepackage[textsize=tiny]{todonotes}

\title{Laws of Learning Dynamics and the Core of Learners}
\author{
	Inkee Jung\thanks{Equal contribution. Correspondence: \texttt{inkeej@bu.edu}.}
	\quad
	Siu Cheong Lau\thanks{Equal contribution. Correspondence: \texttt{scllouis@bu.edu}.}\\
	Department of Mathematics, Boston University\\
	Boston, USA
}
\date{}

\begin{document}

\maketitle

\begin{abstract}
We formulate an algorithm-independent theory of learning dynamics for individual models and ensembles. We establish a conservation law, the vanishing of predictive entropy under successful learning, quantitative entropy and core-expansion bounds, and calibration and finite-sample validation results for identifying reliable predictions. These results motivate a two-generation ensemble that uses Jensen--Shannon (JS) disagreement to route uncertain inputs to specialized learners. CIFAR-10 experiments examine controlled perturbations, distribution shifts, and adversarial attacks. On a heterogeneous test environment, routing improves accuracy over probability averaging while keeping core accuracy near $95\%$. For adversarial defense, we consider a query-only black-box setting: the attacker can query the overall prediction function but does not know a particular internal decomposition into learners and routing rules. Our current attacks use a fixed-initial-route gradient, which is accurate for a single-step attack when the route remains unchanged, but may deviate from actual multi-step attacks, especially for large perturbations. We propose that adding further adversarially trained layers may improve resistance to stronger multi-step attacks; we leave this possibility for future investigation.
\end{abstract}

\begin{center}
\end{center}

\section{Introduction}
In this paper, we formulate learning dynamics for both individual models and ensembles, and establish two algorithm-independent laws: a conservation identity and the asymptotic vanishing of predictive entropy under successful learning, formally analogous to the First and Third Laws of Thermodynamics. We strengthen the latter with a quantitative entropy--cross-entropy inequality and a theorem bounding the measure outside the low-entropy core of an ensemble; in particular, this complement tends to zero under successful learning with a uniform bound on the number of active learners. We also prove calibration-based accuracy bounds and a finite-sample validation guarantee for JS-selected cores of a fixed ensemble, distinguishing low model disagreement from empirically verified correctness. These results motivate a two-generation logifold that routes inputs rejected by an initial committee to newly trained specialists. Our CIFAR-10 experiments test JS-based core selection under controlled perturbations, previously unseen semantic and corruption shifts, and adversarial stress tests. On a heterogeneous mixture of clean and perturbed inputs, the routed system attains $86.83\%$ mean global accuracy, versus $78.44\%$ for ungated probability averaging, while maintaining core accuracy near $95\%$. For adversarial defense, our intended adversary has query-only black-box access to the overall prediction function, not to its constituent learners or routing implementation. Our attack experiments approximate this setting using fixed-initial-route gradients, which may deviate from those of the routed predictor during multi-step attacks.

The laws are asymptotic: they do not require predictive entropy to decrease at every training step, nor do they depend on any particular optimization algorithm. In applications, increased disagreement can signal environmental changes or unreliable predictions, although not every error or shift must increase disagreement. We therefore use disagreement as a screening signal, not as proof of correctness: accepted regions are selected using labeled validation data, and detected out-of-core inputs motivate new specialists. For a trained ensemble held fixed, a sufficiently large increase in its average entropy relative to a stationary baseline can likewise signal a change in the data environment. Motivated by these quantities, we construct an ensemble learning system consisting of layers of models that can detect and divert some adversarial attacks through entropy-based routing.

Learning dynamics is a very rich and foundational subject, see for instance \cite{Markov_decision_processes, NNLearning,Nonlinear_systems, Dynamic_programming,  loss_surfaces, saddle,SGD_Bayesian,  mean_field, bias, convergence_generalization}. We will take a concise, mathematical and practical approach.

Entropy and cross entropy play crucial roles in machine learning and data science, providing a quantitative framework for measuring uncertainty and loss in predictive modeling, see for instance \cite{Hendrycks2017, Guo2017,decomposition_uncertainty}. The concept of entropy has a long history, originating in 19th-century thermodynamics. The formulation most directly relevant to data science is Shannon entropy, introduced in his mathematical theory of communication \cite{Shannon}. 

A model $(\Phi,f)$ for the truth function $\cT$ is defined as a composition of two functions: the Euclidean representation $\Phi$ that maps a part of the full domain to an Euclidean space (called a feature space), and the decision function $f$ that maps the feature space to a subset of the truth labels. 
We study a modular AI architecture \cite{society_of_mind,neural_module} in which an ensemble is composed of heterogeneous models acting on different input modalities and output subspaces. This puts modular ensembles, Mixture-of-Experts \cite{adaptive_MOE} and multimodal learning \cite{multimodal,multimodal_Boltzmann,multimodal_survey} into a unified framework, which have the advantages of reduced catastrophic forgetting \cite{Google_nested}, improved accuracy \cite{conditional_computation,Sparsely_Gated_MOE,scaling_giant_conditional,switch_transformers,transferable_visual_language}, and finer control over training through targeted fine-tuning and model-level error attribution \cite{ST_MoE}.

The entropy of an ensemble, as formulated in this work, consists of two principal components: the total entropy over the union of the domains covered by the constituent models, and a measure of the complement of this union within the full domain of the underlying ground-truth function.
This formulation reveals two fundamental ingredients of the learning process. The first is reinforcement and refinement \cite{stat_grad_RL,RL_book,  policy_grad_RL, human_deep_RL, trust_policy_optimization, ProximalPO, soft_actor_critic} within the union of domains already covered by existing models. The second is the expansion of the knowledge domain \cite{learning_domains,domain_inv_feature,domain_adv_train,lost_domain} through the training of new models on newly observed or generated data.

We implement this theory of learning process to locate untrustworthy predictions and to develop a defense mechanism against common types of adversarial attacks \cite{AdvGoodfellow,AdvCW,APGD,auto_attack}. The idea is to construct a logifold \cite{Jung-Lau2,Jung-Lau} that consists of several layers of models (also called generations). We illustrate by using two layers for simplicity. The first layer contains two types of models: those trained on original samples and those trained on weakly perturbed samples. The second layer consists of models trained on incoming (possibly strongly) perturbed samples detected by the first layer. In the experiments, we use Jensen--Shannon disagreement in the first layer to identify inputs to be passed to the second layer. Beyond the static rejection or abstention mechanisms proposed in prior work \cite{CARL, Pang2022, CPR}, our system treats rejected samples as active signals to evolve into a newly immunized generation.

In this design, the first layer functions as an immunized system that can flag some adversarial perturbations through JS disagreement. Once a non-negligible volume of inputs is detected outside the core, the system transfers models from the first layer to a second layer, where additional training is performed using the detected adversarial samples. 

An attack on particular sub-models may exploit or change the routing mechanism. Our design is motivated by the possibility that attacking models in one layer increases their disagreement, causing the perturbed input to leave that layer's core and be passed to a subsequent layer with different specialists. The same mechanism could be repeated with further generations. This is an empirical design principle rather than a universal implication: an attack may also induce confidently wrong, mutually agreeing predictions, which a disagreement-only gate need not reject. However, a sufficiently diverse collection of models within each layer may reduce the likelihood of such correlated errors, making disagreement-based detection more effective. Our experimental committee combines different architectures (VGG and ResNet) with specialists trained on different perturbation types, providing complementary opportunities for errors to appear as disagreement. This does not rule out shared failures, which is why core accuracy is checked against labeled validation data rather than inferred from agreement alone.

Our adversarial information model is \emph{query-only black box}: an attacker can query the overall prediction function $F:X\to\Delta(Y)$ and observe its probability outputs, but is not given the feature maps, constituent models, parameters, or routing implementation. A function may have different \emph{presentations} by internal models and routing rules, so its input--output behavior does not specify a privileged decomposition. Where the routing decision is locally constant, gradients of $F$ could in principle be approximated by finite differences of its queried outputs. Instead of performing such a query-based gradient estimate, our stress tests use analytic gradients of the branch selected at the original input and keep that branch fixed during multistep APGD optimization. Thus our experiment is an idealized \emph{fixed-initial-route gradient approximation} to the query-only threat model. The approximation is more faithful for a few sufficiently small attack steps that remain within the initial routing region. Over many steps, the attack trajectory has more opportunities to cross a routing boundary, after which the fixed-route gradients may deviate substantially from those of the complete predictor. Where the selected route remains unchanged, however, the surrogate agrees with the actual predictor and its exact branch gradients avoid query-estimation noise. The resulting perturbed inputs are evaluated through the complete routing rule, so the reported accuracy measures the behavior of the full system on these attacks rather than of the fixed-branch surrogate. An attack that queries the complete routed predictor at each step can instead account for changes in the selected route; see Section~\ref{subsec:adversarial_results}.

Our present defense uses two layers and is evaluated against the fixed-initial-route approximate attacks described above. Additional layers trained on adversarial samples reaching the out-core of preceding layers may improve protection when attacks cause routing changes, but this remains a design hypothesis rather than an experimental conclusion. Future experiments will evaluate full query-only, route-adaptive multi-step attacks and test whether deeper defensive layers improve resistance to them.
We evaluate the proposed approach on the CIFAR-10 \cite{CIFAR-10} dataset under gradient-based attacks such as APGD and AutoAttack \cite{auto_attack, APGD}. In the controlled pollution and diverted-attack evaluations, the routed logifold often improves accuracy over the corresponding constituent ensembles and a naive probability average. These evaluations do not constitute a certificate of robustness against an adaptive attack on the complete routing system. They nevertheless quantify the effect of routing on global accuracy and accepted-core coverage under the specified attack protocols.

\section*{Related works}
An interesting recent work \cite{Okanohara} studies a thermodynamic theory of Learning in Fokker-Planck Dynamics. It defines an ensemble as a probability distribution of model configurations and deduces inequalities for the entropy of this probability distribution under stochastic gradient descent. Our perspective is fundamentally different. We concern about the \emph{data space} and use models as charts to cover the data space; the ensemble entropy we concern is an integration over the data space. Our ensemble entropy measures disagreements and uncertainty in the \emph{responses} of models rather than their configuration parameters. 

The distinction is substantial. For instance, models with different weights can give exactly the same responses to a set of inputs; in this case, the entropy for model configurations is drastically different from the entropy of model responses. Moreover, our framework remains meaningful even in the case of a single model, where the entropy of model responses continues to capture important information about uncertainty and learning dynamics. By contrast, the entropy of a distribution over model configurations is trivially zero when only a single model is present.

A complementary information-theoretic approach is developed by Schweighofer et al.~\cite{schweighofer2025}, who study measures of predictive uncertainty and clarify how their interpretation depends on the predictive model and the approximation to the underlying distribution. Their results also highlight limitations of disagreement-based uncertainty measures for detecting distribution shifts. Our focus is different: we integrate predictive entropy and Jensen--Shannon disagreement over data domains, allow for regions not yet covered by a learner, and establish quantitative relationships with learning error and the sizes of low-entropy cores. In particular, agreement among models is not by itself taken as evidence of correctness; this motivates the calibration and empirical validation results in our framework.

The reliable-core construction is also related to selective classification. Geifman and El-Yaniv~\cite{Geifman2017} study selective prediction for deep neural networks, in which a classifier abstains on some inputs to achieve controlled risk at a prescribed coverage, with high-probability guarantees. Their risk--coverage viewpoint is closely related to our distinction between core accuracy and core coverage, and to our finite-sample validation bound for JS-selected regions of a fixed ensemble. A main difference is the role of the rejected inputs: rather than treating abstention as the final outcome, our multi-generation architecture uses out-of-core samples to train additional specialists and expand the domain covered by reliable predictions. Our bounds concern the specified sampling and calibration settings and do not, on their own, certify adversarial robustness.

\subsection*{Declaration of the use of AI}
AI tools, particularly ChatGPT, were used to assist in reviewing and revising this manuscript, improving the exposition, and identifying areas for mathematical refinement. Through iterative discussions with AI, several existing results were strengthened and more refined propositions and quantitative estimates were formulated and incorporated into the paper. The authors critically evaluated the AI-assisted suggestions and retain full responsibility for the mathematical arguments, experimental results, and conclusions.

\section{Laws of Learning for a model}
Let 
$$\cT:X \to Y$$ 
be a measurable function that represents the truth where $X$ is a topological measure space with $\mu(X) < +\infty$ and $1<|Y|<\infty$. We identify each label $y\in Y$ with the corresponding one-hot vector $\delta_y\in\Delta(Y)$ whenever it appears in a formula involving probability distributions. In other words, we assume that the actual world we want to model has no fuzziness.

\begin{defn} \label{def:model}
	A model $(\Phi,f)$ for the truth function $\cT$ is a continuous measurable map $\Phi: S \to \R^n$ for a measurable subset $S \subset X$ and a function 
	$$f=(f_1,\ldots,f_t): \R^n \to \Delta(T)$$ 
	which is continuous on the image of $\Phi$, and $\Delta(T)$ is the simplex of probability distributions on a finite subset $T \subset Y$ of $t$ elements with $T \supset \cT(S)$. $\Phi$ is called to be an Euclidean representation and $f$ is called to be a decision function.
\end{defn}

\begin{defn}[Entropy of a model] \label{def:entropy}
	Let
	\begin{equation}
		H:\Delta(T) \to \R_{\geq 0}, \,\,H(y) = -\sum_{a=1}^t y_a \log_t y_a \label{eq:H(y)}
	\end{equation} 
	be the Shannon entropy function, normalized by the number $t=|T|$ of labels of this model (with $0\log 0:=0$). For $t=1$, we use the convention $H(1)=0$.
	The total entropy of a model $(\Phi,f)$ for $\cT$ is 
	$$ H(\Phi,f) := \int_{\Dom(\Phi)} H(f(\Phi(x))) dx + \mu(X-\Dom(\Phi)). $$
\end{defn}
The second term assigns the value $1$ to uncovered inputs. For $t\geq 2$, this is the maximal normalized entropy of a $t$-label model, $\sum_{a=1}^{t}(-\frac1t\log_t\frac1t)=1$. 

It is well-understood that the entropy quantifies the uncertainty of the model $f$. In other words,
\begin{center}
	{\bf Entropy is fuzziness.}
\end{center} 
Note that $H(\Phi,f)$ does not depend on the truth function $\cT$: we do not need to know about $\cT$ in order to compute the entropy $H(\Phi,f)$.

Next, we consider the distance between $f$ and the truth function $\cT$ which quantifies the error of the model $f$. 

\begin{defn}[Cross entropy between the truth and a model] \label{def:cross}
	Let
	$$ h_t: \Delta(T) \times \Delta(T)^\circ \to \R_{\geq 0},\,\, h_t(y,\gamma) = -\sum_{a=1}^t y_a \log_t \gamma_a $$
	be the cross entropy function for a model with $t=|T|\geq2$ labels, where $\Delta(T)^\circ$ denotes the relative interior of the simplex. For $t=1$, we set $h_1(1,1)=0$.
	The cross entropy between the truth function $\cT$ and a model $(\Phi,f)$ for $\cT$ is defined to be $H(\cT,(\Phi,f)) = \int_S h_t(\delta_{\cT(x)}, f(\Phi(x))) \,dx + \mu(X-S)$, which equals
	$$ H(\cT,(\Phi,f)) := \int_S -\log_t \left(f_{\cT(x)}(\Phi(x))\right)  dx + \mu(X-S)$$
	where $S = \Dom \left(\Phi\right)$ and $f_{y}(s) := 0$ if $y \not\in T$.
\end{defn}
The second term assigns uncovered inputs the reference penalty $1$, equal to $-\log_t(1/t)$ for a uniform $t$-label prediction when $t\geq2$; cross entropy itself is unbounded above.
The first term is $+\infty$ if the truth label is assigned probability zero on a set of positive measure. The condition $\cT(S) \subset T$ in Definition~\ref{def:model} ensures that the truth labels belong to the model target set, but finite cross entropy additionally requires positive probability at the true label almost everywhere.

The difference between cross entropy and entropy is quantified by the following proposition. The left hand side is the total work done needed to go from $(\Phi,f)$ to $\cT$ against the gradient vector field of $H$ that measures fuzziness. Interpreting the cross entropy $H(-,(\Phi,f))$ as potential, then the right hand side is the potential difference between $\cT$ and $(\Phi,f)$. Thus, the identity is an analog of {\bf the conservation law (The First Law of Thermodynamics)}: 
\begin{center}
	{\bf Effort paid in learning the truth $\cT$ turns to potential.}
\end{center}

\begin{prop}[Conservation Law of Learning] \label{prop:PD}
	Assume that $f(\Phi(x))$ lies in the relative interior of $\Delta(T)$ almost everywhere on $\Dom(\Phi)$ and that the displayed integrals are defined.
	\begin{align*}
		&\int_{\Dom(\Phi)} (\mathrm{grad}\, H)|_{f(\Phi(x))} \cdot (\cT(x) - f(\Phi(x))) \,dx\\
		=& H(\cT,(\Phi,f)) - H(\Phi,f).
	\end{align*}
\end{prop}

\begin{proof}
    For a tangent vector $y'$ to the simplex $\Delta(T)$, we have
    \[
       \mathrm{d}H_y(y')=-\sum_a y'_a\log_t y_a,
    \]
    since $\sum_a y'_a=0$. Here the right-hand side is the linear
    extension of cross entropy in its first argument. Applying this with
    $y=f(\Phi(x))$ and $y'=\cT(x)-f(\Phi(x))$ gives
    \[
       \mathrm{d}H_y(y')=h_t(\delta_{\cT(x)},y)-H(y).
    \]
    Integrating over $\Dom(\Phi)$ gives the asserted identity; the
    uncovered-domain terms cancel.
\end{proof}

Since we work independently of any particular learning algorithm, we define a learning process by its progress toward the truth, measured through decreasing cross entropy, rather than by a prescribed training procedure.

\begin{defn}\label{def:learning_process}
	A learning process of a model for $\cT$ is a sequence of models $(\Phi^{(i)},f^{(i)})_{i=1}^N$ for $N \in \Z_{>1} \cup \{+\infty\}$ such that $\mathrm{Dom}(\Phi^{(i)}) \subset \mathrm{Dom}(\Phi^{(j)})$ and $H(\cT,(\Phi^{(i)},f^{(i)})) > H(\cT,(\Phi^{(j)},f^{(j)}))$ for $i<j$.
	
	The learning process is said to be successful if
	$$\lim_{i\to\infty} H(\cT,(\Phi^{(i)},f^{(i)})) = 0.$$
\end{defn}

\begin{rem}
	While Definition~\ref{def:learning_process} is stated in terms of a monotone sequence of errors measured by cross entropy, in practice one may select a subsequence of training checkpoints with decreasing cross entropy and regard it as the learning process. 
	
	In the definition, the domain can either be expanding or remaining the same (which is the typical meaning of the subset notation).
	
	We define a learning process to be successful if the error measured by cross entropy tends to zero. In finite training, a small cross entropy is an idealized target that may be approximated by a suitable stopping criterion.
\end{rem}

The \textbf{Law of Learning Dynamics} is the asymptotic statement:
\begin{center}
	\textbf{Entropy tends to zero in a successful learning process.}
\end{center}
This has interesting similarity with the Third Law of Thermodynamics for crystallization.

\begin{prop}[Quantitative entropy bound]\label{prop:quant_entropy}
    Let $(\Phi,f)$ be a model for $\cT$ with $t=|T|\geq2$ labels, and let
    $M=\mu(X)\in(0,\infty)$. For $0\leq q\leq1$, write
    \[
       h_2(q)=-q\ln q-(1-q)\ln(1-q),
    \]
    with the convention $0\ln0=0$. Define, for $s\geq0$,
    \[
       q_t(s)=1-t^{-\min\{s,1\}},\qquad
       \omega_t(s)=\frac{h_2(q_t(s))+q_t(s)\ln(t-1)}{\ln t}.
    \]
    Then, writing $E=H(\cT,(\Phi,f))$, we have
    \begin{equation}\label{eq:quant_entropy_bound}
       H(\Phi,f)\leq M\,\omega_t\!\left(\frac{E}{M}\right).
    \end{equation}
    Here $\omega_t(+\infty):=1$ if $E=+\infty$. Moreover, as $s\to0^+$,
    \[
       \omega_t(s)=s\ln\frac1s+O_t(s).
    \]
    In particular, as $E\to0^+$ (with $M$ fixed),
    \begin{equation}\label{eq:quant_entropy_rate}
       H(\Phi,f)\leq E\ln\frac{M}{E}+O_t(E).
    \end{equation}
    For $t=1$, the two total entropies equal
    $\mu(X\setminus\Dom(\Phi))$ under our conventions.
\end{prop}

\begin{proof}
    Fix $x\in\Dom(\Phi)$, and write $p=f(\Phi(x))$,
    $q=1-p_{\cT(x)}$, and $e=-\log_t p_{\cT(x)}$.
    For fixed $q$, the entropy of $p$ is maximized when the other
    $t-1$ coordinates are equal. Hence
    \[
       H(p)\leq\frac{h_2(q)+q\ln(t-1)}{\ln t}.
    \]
    If $e\leq1$, then $q=1-t^{-e}=q_t(e)$, so
    $H(p)\leq\omega_t(e)$. If $e>1$, the same bound holds because
    $H(p)\leq1=\omega_t(e)$.

    Extend the pointwise entropy and cross entropy to
    $X\setminus\Dom(\Phi)$ by the value $1$, as in Definitions
    \ref{def:entropy} and \ref{def:cross}. Since $\omega_t(1)=1$,
    the pointwise bound holds on all of $X$.
    The function
    \[
       q\longmapsto\frac{h_2(q)+q\ln(t-1)}{\ln t}
    \]
    is concave and non-decreasing for $0\leq q\leq1-1/t$;
    $q_t(s)$ is concave and non-decreasing on $[0,\infty)$.
    Thus $\omega_t$ is concave. Jensen's inequality gives, for finite $E$,
    \[
       H(\Phi,f)\leq\int_X\omega_t(e(x))\,d\mu(x)
       \leq M\,\omega_t\!\left(\frac1M\int_Xe(x)\,d\mu(x)\right)
       =M\,\omega_t\!\left(\frac{E}{M}\right).
    \]
    For $E=+\infty$, the asserted bound is simply
    $H(\Phi,f)\leq M$. Finally,
    $q_t(s)=s\ln t+O_t(s^2)$ and
    $h_2(q)=q\ln(1/q)+O(q)$ as $q\to0^+$,
    which proves the stated asymptotic estimate.
\end{proof}

\begin{thm} \label{thm:law_single}
    Suppose that $(\Phi^{(i)},f^{(i)})_{i=1}^\infty$ is a learning process
    such that $\lim_{i\to\infty}H(\cT,(\Phi^{(i)},f^{(i)}))=0$.
    Then
    \[
       \lim_{i\to\infty}H(\Phi^{(i)},f^{(i)})=0.
    \]
    In particular, there exists a learning subprocess
    $(\Phi^{(i_j)},f^{(i_j)})_{j=1}^\infty$ along which the entropy
    decreases to zero (non-increasingly if it becomes zero).
\end{thm}

\begin{proof}
    This is a direct corollary of Proposition~\ref{prop:quant_entropy}.
    Put $E_i=H(\cT,(\Phi^{(i)},f^{(i)}))$ and
    $t_i=|T_i|\leq|Y|$. If $\mu(X)=0$, the conclusion is immediate.
    Otherwise, for $t_i\geq2$ the proposition gives
    \[
       H(\Phi^{(i)},f^{(i)})
       \leq\mu(X)\,\omega_{t_i}\!\left(\frac{E_i}{\mu(X)}\right).
    \]
    Since $t_i$ ranges over the finite set $\{2,\ldots,|Y|\}$
    and $\omega_t(s)\to0$ as $s\to0^+$, the right-hand side
    tends to zero. If $t_i=1$, the proposition's final observation gives
    $H(\Phi^{(i)},f^{(i)})=E_i\to0$.
    The claimed subprocess follows by extracting a non-increasing
    subsequence from the non-negative sequence of entropies converging
    to zero.
\end{proof}

\Needspace{12\baselineskip}
\begin{prop}[Calibration and correctness of a model]
\label{prop:single_calibration}
Let $(X,\mu)$ be the data space and let $(\Phi,f)$ be a model for
$\cT:X\to Y$, with domain $S=\Dom(\Phi)$ of positive finite measure
and target set $T\subseteq Y$ of size $t\geq2$.
Write $p(x)=f(\Phi(x))\in\Delta(T)$, put
$\nu_S=\mu|_S/\mu(S)$, and let
$F(x)=\argmax_{a\in T}p_a(x)$ with a fixed tie-breaking rule.
Assume the model is \emph{calibrated} on $S$, that is,
\[
  \mathbb E_{\nu_S}\!\left[
   \mathbf 1_{\{\cT=a\}}\mid p\right]=p_a
   \quad (a\in T).
\]
For any measurable $A\subseteq S$ determined by $p$ with $\mu(A)>0$,
let $\operatorname{Acc}_A(F)=\mu(\{x\in A:F(x)=\cT(x)\})/\mu(A)$.
Then
\begin{equation}\label{eq:single_calibrated_accuracy}
 \operatorname{Acc}_A(F)
 =\frac1{\mu(A)}\int_A\max_{a\in T}p_a(x)\,d\mu(x)
 \geq t^{-\frac1{\mu(A)}\int_A H(p(x))\,d\mu(x)}.
\end{equation}
In particular, the entropy core
$A_h=\{x\in S:H(p(x))\leq h\}$, if of positive measure,
satisfies $\operatorname{Acc}_{A_h}(F)\geq t^{-h}$.
\end{prop}

\begin{proof}
Since $A$ and $F$ are determined by $p$, calibration gives
\[
 \int_A\mathbf 1_{\{F(x)=\cT(x)\}}\,d\mu(x)
   =\int_A\max_{a\in T}p_a(x)\,d\mu(x).
\]
For a probability vector $q\in\Delta(T)$,
\[
 H(q)=-\sum_aq_a\log_tq_a\geq-\log_t\max_aq_a,
 \qquad\text{so}\qquad \max_aq_a\geq t^{-H(q)}.
\]
The function $u\mapsto t^{-u}$ is convex. Jensen's inequality
on $A$ gives the claimed lower bound, and the assertion for
$A_h$ follows immediately.
\end{proof}

\section{Law of Learning for an ensemble}
\begin{defn}
	An ensemble is a collection of models $\cU = \{(\Phi_l,f_l): l=1,\ldots,K\}$. Its knowledge domain is the union $\Dom(\cU) = \bigcup_{l=1}^K \Dom(\Phi_l)$.
\end{defn}

\begin{defn}[Jensen-Shannon Divergence for Ensembles] \label{def:JS}
	Let $\cU$ be an ensemble of models.  
	For each $x \in X$, let $\{(\Phi_l,f_l): l=1,\ldots,K_x\}$ be the collection of models in $\cU$ whose domains contain $x$.  
	Regard all model predictions as elements of $\Delta(Y)$ by extending them by zero outside their target sets. Let $\vec{y}_l = f_l(\Phi_l(x))$ be the prediction of each active model $l$ at $x$, and let $\vec{y}_0 := \frac{1}{K_x} \sum_j \vec{y}_j$ be the average prediction.  
	We have the Jensen-Shannon divergence at $x$ defined as
	$$ JS(x) = H_Y(\vec{y}_0) - \frac{1}{K_x} \sum_j H_Y(\vec{y}_j), $$
	where $H_Y(\vec{y}) = -\sum_i y_i \log_{|Y|} y_i$ is the normalized entropy on the common label simplex $\Delta(Y)$ (distinct from the per-model normalization $H$ when $|T|<|Y|$).  
	(If $x$ is not in the domain of any model, we set $JS(x) = 0$.)  
	The total Jensen-Shannon divergence is defined to be
	$$ JS(\cU) = \int_{x \in X} JS(x) dx. $$
	
	Given two collections of distributions $Y=(\vec{y}_j : j=1,\ldots,n_1)$ and $Z=(\vec{z}_j : j=1,\ldots,n_2)$, the relative Jensen-Shannon divergence is defined as
	$$ RJS(Y||Z) = \frac{1}{n_1n_2} \sum_{j=1}^{n_1} \sum_{k=1}^{n_2} h_Y(\vec{z}_k,\vec{y}_j) - h_Y(\vec{z}_0,\vec{y}_0) $$
	where $\vec{y}_0 = \frac{1}{n_1} \sum_j \vec{y}_j$ and $\vec{z}_0 = \frac{1}{n_2} \sum_j \vec{z}_j$ are the average distributions,  
	and $h_Y(\vec{p},\vec{y}) = -\sum_i p_i \log_{|Y|} y_i \geq 0$ is cross entropy on $\Delta(Y)$ (allowing $+\infty$ at zero probabilities).  
	By convexity of $-\log$, we have $ RJS(Y||Z) \ge 0. $
\end{defn}

\begin{rem}
	$JS(\cU)$ measures the disagreement among models in $\cU$.  
	For instance, when $\cU$ only has one model (or when all models make exactly the same predictions),  
	then $JS(\cU)=0$. Conversely, positive disagreement implies that not all predictive distributions equal the one-hot truth distribution, though it need not imply disagreement in their most probable labels. Agreement alone does not guarantee correctness. This is why we retain predictive entropy alongside JS disagreement and use labeled validation to assess the accuracy of selected cores, rather than assuming that small disagreement establishes correctness.  
	
	We have another quantity, the sum of entropies,  
	\[ \Sigma H(\cU) = \sum_l H((f_l \circ \Phi_l)|_{\Dom(\Phi_l)}) + \mu(X - \Dom(\cU)) \]
	that measures the total uncertainty of the models.  
	We use $(JS,\Sigma H)$ to measure both disagreement and uncertainty. For fixed model domains, the differential identity of Proposition~\ref{prop:PD} applies to the individual entropy terms. The statements proved below concern a conservation identity for $JS$ and asymptotic vanishing of both quantities under successful ensemble learning; they do not assert step-by-step monotonicity.
\end{rem}

\begin{rem}
	$RJS(Y||Z)$ is a measurement of disagreement among the models in $Y$ relative to $Z$.
	For instance, let $\vec{z} = \delta_k$ be a one hot vector. $RJS(Y||\vec{z}) \ge 0$ measures the disagreement of $Y$ at the $k$-th class.
\end{rem}

\begin{rem}
	In an earlier version of this work, we define ensemble entropy using pairwise cross entropies between models in the ensemble.  
	However, cross entropy does not measure disagreement nicely in the following example.  
	Consider $-\frac{1}{n} \log_2 2^{-n^2}$ for a large $n$.  
	This is close to $+\infty$ for large $n$, even though both $\frac{1}{n}$ and $2^{-n^2}$ are close to zero.  
	This example shows that cross entropy between two models could be large even though they almost agree to each other.  
	On the other hand, the Jensen-Shannon divergence must be small when the models are close to each other.  
	Thus we decide to use the Jensen-Shannon divergence to measure disagreement among models in an ensemble.
\end{rem}

\begin{thm}[Vanishing of entropy]\label{thm:van_ent}
	For an ensemble of models $\cU = \{(\Phi_l,f_l): l=1,\ldots,K\}$, 
	\[JS(\cU)=0 \textrm{ and } \Sigma H (\cU) = 0\] 
	if and only if the models cover the whole space in measure sense,  
	that is $X - \bigcup_{l=1}^K \Dom(\Phi_l)$ has measure zero,  
	and there exists a function $F: X \to Y$ such that for every $l = 1,\ldots,K$, $F|_{\Dom(\Phi_l)} = f_l \circ \Phi_l$ almost everywhere. 
\end{thm}

\begin{proof}	
	The direction $\Leftarrow$ is straightforward since it implies that there is no uncertainty nor disagreement in the ensemble.  
	We prove the other direction $\Rightarrow$.  
	if the models do not cover the whole space in measure sense,  
	then $\Sigma H(\cU) > 0$ since there exists a subset of positive measure that is not covered by any model,  
	over which the pointwise entropy is $1$ by definition.  
	Thus $X - \bigcup_{l=1}^K \Dom(\Phi_l)$ has measure zero.  
	Since entropy is non-negative,  
	$\Sigma H (\cU) = 0$ implies that $H(f_l(\Phi_l(x))) = 0$ for almost every $x \in \Dom(\Phi_l)$ for every $l$.  
	Since the image of $f_l$ lies in the simplex $\Delta(T_l)$,  
	this implies $f_{l,a}(\Phi_l(x)) = 1$ for exactly one value of $a \in T_l$ and $0$ otherwise.  
	Thus for every $l$, $f_l\circ \Phi_l$ gives a function on the complement of a measure zero set in $\Dom(\Phi_l)$.  
	Now we consider $JS(\cU)$, which is also an integral of a non-negative function.  
	$JS(\cU) = 0$ implies that the pointwise Jensen-Shannon divergence is zero almost everywhere.  
	This implies that for almost every $x$,  
	the predictions of all models whose domains contain $x$ agree.  
	This implies that the functions defined by $f_l \circ \Phi_l$ agree on the intersection of their domains,  
	and hence glue to a global function $F$.
\end{proof}

On the covered domain, writing $t_l=|T_l|$ and $K_x$ for the number of active models, we have
\[H_Y(\bar{\cU}(x))=JS_x(\cU)+\frac{1}{K_x}\sum_{l:\Dom(\Phi_l)\ni x}\frac{\ln t_l}{\ln|Y|}H(f_l(\Phi_l(x))),\]
where the summand is $0$ for $t_l=1$. The scaling factor is needed because individual model entropies use base $t_l$, while JS uses base $|Y|$.  
Below, we use $H_Y(\bar{\cU}(x))$ and $\Sigma H_x(\cU)$ to define a core of an ensemble,  
which is the part of dataset that the ensemble has high certainty and low disagreement.

\begin{defn} \label{def:core}
	Given $p \in (0,1)$, the core of an ensemble of models $\cU = \{(\Phi_l,f_l): l=1,\ldots,K\}$ in threshold $p$ is defined to be the subset
	$$ C_p(\cU) = \left\{x \in X: \max\{H_Y(\bar{\cU}(x)) , \Sigma H_x(\cU)\} < p \right\} \subset X $$
	where $\bar{\cU}(x)=K_x^{-1}\sum_{l:\Dom(\Phi_l)\ni x}f_l(\Phi_l(x))$ is the average of the active predictions and $\Sigma H_x(\cU)=\sum_{l:\Dom(\Phi_l)\ni x}H(f_l(\Phi_l(x)))$ on $\Dom(\cU)$. On its complement, both quantities in the maximum are assigned the value $1$, so uncovered points do not lie in the core.
\end{defn}

\begin{prop}[Non-fuzzy limit of the core]\label{prop:NonFuzLimCore}
	Let $\cU$ be an ensemble with $K=\#\{l:\Dom(\Phi_l)\neq\emptyset\}\geq1$ active models (and possibly additional inactive models). Consider its core $C_p(\cU)$ where
	\[ 0 < p < \min \left\{\frac{1}{|Y|}, - a \log_{|Y|} a, \frac{1}{2K} \log_{|Y|} 2K \right\} \] 
	for $a = \max \left\{\frac{1}{2}, \frac{|Y| K - (|Y|-1)}{|Y| K}\right\}$.
	For every $x\in C_p(\cU)$ and $l$ such that $\Dom(\Phi_l) \ni x$,  
	there exists a unique $a=a_{l,x} \in T_l$ that $f_{l,a}\circ\Phi_l(x)\in [0,1]$ achieves maximum;  
	moreover, $a_{l,x}$ are the same for all such $l$.  
	Thus, the ensemble defines a function 
	$$F: C_p(\cU) \to Y$$ 
	which is called to be the non-fuzzy limit of the ensemble restricted to the core.
\end{prop}

\begin{proof}
Recall that $- a \log_{|Y|} a$ is a concave function on the interval $[0,1]$ whose maximum is achieved at $1/e < 1/2$ and its value at $a = 1/|Y|$ is $1/|Y|$.  

First, we show that each model prediction $f_l \circ \Phi_l(x)$ for $x \in C_p(\cU)$ has a unique maximum entry.  
We have, with $t_l=|T_l|\leq|Y|$,
\[ H_Y(f_l \circ \Phi_l(x))\leq H(f_l \circ \Phi_l(x)) \leq \Sigma H_x(\cU) < \min\left\{\frac{1}{|Y|}, \frac{1}{2}\log_{|Y|} 2\right\}. \]
Then for each $i$, $-(f_l \circ \Phi_l(x))_i \,\log_{|Y|} (f_l \circ \Phi_l(x))_i < \min\left\{\frac{1}{|Y|}, \frac{1}{2}\log_{|Y|} 2\right\}$,  
which implies that $(f_l \circ \Phi_l(x))_i < \frac{1}{|Y|}$ or $(f_l \circ \Phi_l(x))_i > \frac{1}{2}$.  
Since the sum of all entries of $f_l \circ \Phi_l(x)$ is $1$, it cannot happen that all its entries are less than $\frac{1}{|Y|}$.  
Therefore, there must be at least one entry greater than $\frac{1}{2}$, and there can only be at most one such entry which is the maximum entry.  

Next, we show that all model predictions $f_l \circ \Phi_l(x)$ for $x \in C_p(\cU)$ have the same unique maximum entry.  
This is trivial if $K = 1$. Let $K > 1$.  
Let $i$ be the unique maximum entry of $f_l \circ \Phi_l(x)$ for some $l$ such that $\Dom(\Phi_l) \ni x$.  
Suppose there exists another model $r$ such that $\Dom(\Phi_r) \ni x$ whose unique maximum entry of $f_r \circ \Phi_r(x)$ is different from $i$.  
Then $(f_r \circ \Phi_r(x))_i < \frac{1}{|Y|}$ by above.  
Then the average $\bar{y}$ of the $i$-th entries of all model predictions at $x$ satisfies
$$ \frac{1}{2K} \leq \bar{y} \leq \frac{\frac{1}{|Y|} + (K-1)}{K} = \frac{|Y|K - (|Y|-1)}{|Y|K}. $$
The least value is achieved when only one model has the $i$-th entry equal to $1/2$ and the rest have it equal to $0$,  
and the greatest value is achieved when one model has the $i$-th entry equal to $1/|Y|$ and the rest have it equal to $1$.
But this implies 
\[
-\bar{y}\log_{|Y|}\bar{y}
\geq\min\left\{\frac{1}{2K}\log_{|Y|}(2K),
                   -a\log_{|Y|}a\right\}>p.
\]
This violates the condition for the core on the entropy of the average.
\end{proof}

We have the following conservation law. We write
$RJS(\cU\Vert\cT)$ for the integral over $\Dom(\cU)$ of the
pointwise relative Jensen--Shannon divergence between the active
predictions and the singleton truth distribution $\delta_{\cT(x)}$.

\begin{prop}[Conservation law for Jensen--Shannon divergence]
\label{prop:js_conservation}
Assume that, for almost every $x\in\Dom(\cU)$, each active prediction
lies in the relative interior of its own target simplex $\Delta(T_l)$,
and that the displayed quantities are integrable. Write
$\vec y_l(x)=f_l(\Phi_l(x))$ for the $K_x$ predictions defined at $x$, and set
\[
  J_x(s)=JS\bigl((1-s)\vec y_l(x)+s\delta_{\cT(x)}:
               \Dom(\Phi_l)\ni x\bigr),\qquad 0\le s<1.
\]
Then
\begin{equation}\label{eq:js_conservation}
 -\int_{\Dom(\cU)}J_x'(0)\,d\mu(x)
 =JS(\cU)+RJS(\cU\Vert\cT).
\end{equation}
\end{prop}

\begin{proof}
Fix a covered point $x$, and write $n=K_x$,
$\vec y_0=n^{-1}\sum_{l=1}^{n}\vec y_l$, and
$\vec T=\delta_{\cT(x)}$. As $\cT(x)\in T_l$ for each model
active at $x$, the line segment from $\vec y_l$ toward $\vec T$
lies in $\Delta(T_l)$. Thus the derivative at $s=0$ exists
within the corresponding simplex face; no positivity assumption
outside $T_l$ is needed.

For any probability vector $\vec y$ in the relative interior of
its target simplex, the directional derivative of common-base
entropy along $\vec T-\vec y$ is
\[
   dH_Y|_{\vec y}(\vec T-\vec y)
   =h_Y(\vec T,\vec y)-H_Y(\vec y).
\]
Applying this identity to the active vectors and to their mean gives
\begin{align*}
 J_x'(0)
 &=h_Y(\vec T,\vec y_0)-H_Y(\vec y_0)\\
 &\quad-\frac1n\sum_{l=1}^{n}
       \bigl(h_Y(\vec T,\vec y_l)-H_Y(\vec y_l)\bigr)\\
 &=-RJS((\vec y_l)_{l=1}^{n}\Vert(\vec T))
   -JS((\vec y_l)_{l=1}^{n}).
\end{align*}
Integration over $\Dom(\cU)$ proves the proposition.
\end{proof}

The pointwise cross entropy of the truth function $\cT$ and $\cU$ at $x$ is defined as
\begin{equation} \label{eq:H_x(T,U)}
	H_x(\cT,\cU) = \sum_{l: \Dom(\Phi_l)\ni x} h_{t_l}(\delta_{\cT(x)},(f_l\circ \Phi_{l})(x))
\end{equation}
if $x \in \Dom(\cU)$ and $1$ otherwise.
Here $t_l=|T_l|$, so each summand uses that model's own logarithm base; the cross entropy $H(\cT,\cU)$ of $\cT$ and $\cU$ is the integral of this pointwise quantity over $X$. It is an unnormalized sum over active models, not the cross entropy of the average prediction.  

\begin{defn}[Learning process of an ensemble] \label{def:learning_ensemble}
	A learning process of an ensemble of models for $\cT$ is a sequence of ensembles $(\cU^{(i)})_{i=1}^N$ and injective functions 
	$$\mathcal{L}_i:\cU^{(i)} \to \cU^{(i+1)}$$ 
	such that $\Dom(\cU^{(i)}) \subset \Dom(\cU^{(j)})$ and $H(\cT,\cU^{(i)}) > H(\cT,\cU^{(j)})$ for $i < j$. The loss is the sum of constituent cross entropies over the models defined at each point, rather than the loss of their averaged prediction.
	
	A model $(\Phi_l,f_l)$ in an ensemble $\cU$ is called \emph{active} if $\Dom(\Phi_l)\neq\emptyset$, and \emph{inactive} otherwise. The \emph{total number of active models} is
	\[
	  K_{\mathrm{act}}(\cU)
	  :=\#\{l:\Dom(\Phi_l)\neq\emptyset\}.
	\]
	At a point $x\in X$, only models whose domains contain $x$ contribute to the ensemble predictions; their number $K_x$ satisfies $K_x\leq K_{\mathrm{act}}(\cU)$. Inactive models remain formally in the ensemble but do not contribute to its predictions, entropy, or cross entropy.

	The models in $\cU^{(i+1)}$ that are not in the image of $\mathcal{L}_i$ are said to be created in the $i$-th step. A model whose domain becomes empty is said to be annihilated (made inactive) at that step.
\end{defn}

\begin{rem}
The condition $\Dom(\cU^{(i)}) \subset \Dom(\cU^{(j)})$ and $H(\cT,\cU^{(i)}) > H(\cT,\cU^{(j)})$ for $i < j$ ensures that the covered domain does not shrink and that the stated sum of cross entropies decreases; it does not by itself guarantee monotonic improvement in hard-label accuracy. Cross entropy is nevertheless informative about the probability assigned to the true label, and the entropy laws concern this explicitly defined learning quantity; hard-label accuracy is assessed separately in the experiments.  
Otherwise the learning process can be entirely arbitrary.  
We can also instead require that the process has a subsequence that satisfies these conditions.  
Here, we simply take such a subsequence as the learning process (which is also what is done in practice via training with checkpoints).
\end{rem}

We have the following \textbf{Law of Learning Dynamics for ensemble}:

\begin{thm}\label{thm:LLDforEns}
	Let $(\cU^{(i)})_{i=1}^\infty$ be a learning process of an ensemble of models for $\cT$ such that  
	$\lim_{i\to\infty} H(\cT,\cU^{(i)}) = 0$.  
	Moreover, assume that there exists an integer $K\geq1$ such that the total number of active models satisfies $K_{\mathrm{act}}(\cU^{(i)})\leq K$ for every $i$.
	Then 
	\[\lim_{i\to\infty} JS(\cU^{(i)}) = 0 \textrm{ and } \lim_{i\to\infty} \Sigma H(\cU^{(i)}) = 0.\]
	In particular, there exists a learning subprocess $(\cU^{(i_j)})_{j=1}^\infty$ such that  
	both $JS(\cU^{(i_j)})$ and $\Sigma H (\cU^{(i_j)})$ decrease to zero.
\end{thm}

\begin{proof}
    Write $C=|Y|$, $S_i=\Dom(\cU^{(i)})$, and let $K$ bound
    the total number of active models in each ensemble. Since $H(\cT,\cU^{(i)})\to0$, both
    $\mu(X\setminus S_i)\to0$ and the integral of the non-negative sum of
    constituent cross entropies on $S_i$ tends to zero.

    Fix $\eta>0$. Since every constituent has $1\leq t_l\leq C$,
    continuity of the finitely many per-model normalized entropies
    gives a uniform $\delta>0$ such that model cross entropy
    $-\log_{t_l}p_{\cT(x)}<\delta$ implies $H(p)<\eta$
    and $\|p-\delta_{\cT(x)}\|_1<\eta$ (with the zero-loss convention
    for $t_l=1$).
    Let $B_i\subset S_i$ be the set where the sum of constituent cross
    entropies is at least $\delta$. Then
    \[\mu(B_i)\leq H(\cT,\cU^{(i)})/\delta\longrightarrow0.\]
    For $x\in S_i\setminus B_i$, every constituent whose domain contains
    $x$ has cross entropy below $\delta$, so its entropy is below $\eta$;
    there are at most $K_x\leq K_{\mathrm{act}}(\cU^{(i)})\leq K$
    such constituents. Using $H(p)\leq1$ and the uncovered-domain
    convention gives
    \[\Sigma H(\cU^{(i)})\leq
      K\eta\mu(X)+K\mu(B_i)+\mu(X\setminus S_i).\]
    Hence $\Sigma H(\cU^{(i)})\to0$.

    On $S_i\setminus B_i$, every model prediction defined at $x$
    is uniformly close to $\delta_{\cT(x)}$, and so is their average. By continuity of
    entropy, the Jensen--Shannon divergence is uniformly small there.
    Since $0\leq JS_x(\cU^{(i)})\leq1$ and $\mu(B_i)\to0$, integrating
    shows that $JS(\cU^{(i)})\to0$ as well.
\end{proof}

\begin{rem}
	The assumption that the total number of active models in each $\cU^{(i)}$ is bounded by $K$ reflects limited computational resources for models making predictions; formally retained inactive models do not count toward this bound.  
	The boundedness assumption controls the unnormalized sum $\Sigma H$, which counts the uncertainties of models on their nonempty domains. It is not required by the definition of JS disagreement. The present theorem does not make any assumption about how the individual learners are optimized. The bound restricts how many learners are active simultaneously, not how many can be created over time; deactivating older models therefore permits continued specialization under a fixed inference budget.
\end{rem}

\begin{thm}[Expansion of the entropy core]\label{thm:core_expansion}
    Let $C=|Y|\geq2$, and let $\cU$ be an ensemble with at most $K$
    active models, where $K\geq1$ is an integer. For $0<p<1$, consider the
    entropy core $C_p(\cU)$ of Definition~\ref{def:core}. With the functions
    $\omega_t$ from Proposition~\ref{prop:quant_entropy}, define
    \[
      \delta_{p,K,C}
      :=\min_{2\leq t\leq C}\omega_t^{-1}(p/K)>0,
    \]
    where $\omega_t^{-1}$ is the inverse of the strictly increasing
    restriction of $\omega_t$ to $[0,1]$. Then
    \begin{equation}\label{eq:core_expansion_bounds}
      \mu\bigl(X\setminus C_p(\cU)\bigr)
      \leq \min\left\{
       \frac{JS(\cU)+\Sigma H(\cU)}{p},\,
       \frac{H(\cT,\cU)}{\delta_{p,K,C}}
      \right\}.
    \end{equation}
    In particular, if $(\cU^{(i)})_{i=1}^{\infty}$ is a successful
    ensemble learning process with at most $K$ active models at every step, then
    for every fixed $p\in(0,1)$,
    \[
       \mu\bigl(X\setminus C_p(\cU^{(i)})\bigr)
       \leq \frac{H(\cT,\cU^{(i)})}{\delta_{p,K,C}}
       \longrightarrow 0.
    \]
\end{thm}

\begin{proof}
    Write $S=\Dom(\cU)$ and let $K_x$ be the number of models
    whose domains contain $x\in S$, so that
    $K_x\leq K_{\mathrm{act}}(\cU)\leq K$.
    Recall that the entropy $H_Y$ used for Jensen--Shannon
    divergence has logarithm base $C$, while the entropy $H$ of an
    individual model with $t_l$ target labels has base $t_l$.
    By Definition~\ref{def:JS}, on $S$ we have
    \[
       H_Y(\bar{\cU}(x))=JS_x(\cU)
         +\frac1{K_x}\sum_{l:\,x\in\Dom(\Phi_l)}
            H_Y(f_l(\Phi_l(x))).
    \]
    Since $H_Y(f_l(\Phi_l(x)))\leq H(f_l(\Phi_l(x)))$
    (both sides vanish when $t_l=1$), we obtain
    \[
       \max\{H_Y(\bar{\cU}(x)),\Sigma H_x(\cU)\}
       \leq JS_x(\cU)+\Sigma H_x(\cU).
    \]
    The same inequality holds on $X\setminus S$, where the two
    entropies on the left and $\Sigma H_x(\cU)$ are $1$, and
    $JS_x(\cU)=0$, by convention. Markov's inequality therefore gives
    \begin{align*}
       \mu\bigl(X\setminus C_p(\cU)\bigr)
       &\leq \frac1p\int_X
          \max\{H_Y(\bar{\cU}(x)),\Sigma H_x(\cU)\}\,d\mu(x)\\
       &\leq \frac{JS(\cU)+\Sigma H(\cU)}p.
    \end{align*}

    To obtain the bound in terms of cross entropy, let
    $e(x)=H_x(\cT,\cU)$, as in~\eqref{eq:H_x(T,U)}. On $S$, put
    $p_l=f_l(\Phi_l(x))$ and
    $e_l=-\log_{t_l}(p_l)_{\cT(x)}$ for $t_l\geq2$, with $e_l=0$
    when $t_l=1$. Then $e(x)=\sum_l e_l$ and all summands are
    non-negative. The pointwise inequality in the proof of
    Proposition~\ref{prop:quant_entropy} yields
    \[
       \Sigma H_x(\cU)=\sum_l H(p_l)
       \leq\sum_{l:\,t_l\geq2}\omega_{t_l}(e_l)
       \leq K\max_{2\leq t\leq C}\omega_t(e(x)).
    \]
    Let $\bar p=\bar{\cU}(x)$. By convexity of $-\ln$,
    \[
       -\log_C\bar p_{\cT(x)}
       \leq \frac1{K_x}\sum_l-\log_C(p_l)_{\cT(x)}
       =\frac1{K_x}\sum_l\frac{\ln t_l}{\ln C}e_l
       \leq e(x),
    \]
    where the term for $t_l=1$ is interpreted as zero. Applying the
    pointwise bound again, now with $t=C$, gives
    \[
       H_Y(\bar{\cU}(x))\leq \omega_C(e(x)).
    \]
    Thus, whenever $x\in S$ and $e(x)<\delta_{p,K,C}$,
    both $\Sigma H_x(\cU)<p$ and $H_Y(\bar{\cU}(x))<p$,
    so $x\in C_p(\cU)$. Uncovered points have $e(x)=1$ and
    $\delta_{p,K,C}\leq1$. Hence
    \[
       X\setminus C_p(\cU)
       \subseteq\{x\in X:e(x)\geq\delta_{p,K,C}\}.
    \]
    A second application of Markov's inequality gives
    \[
       \mu\bigl(X\setminus C_p(\cU)\bigr)
       \leq\frac{1}{\delta_{p,K,C}}\int_Xe(x)\,d\mu(x)
       =\frac{H(\cT,\cU)}{\delta_{p,K,C}}.
    \]
    The final assertion follows because $\delta_{p,K,C}>0$ is
    independent of the learning step.
\end{proof}

\begin{rem}
    The theorem asserts convergence of core coverage to full measure;
    it does not assert that the cores are nested or that their measures
    increase at every learning step. The result thus guarantees asymptotic expansion in measure without requiring a rigid inclusion of successive cores; retention of particular previously accepted inputs can be addressed separately by the learning policy. More generally, the same conclusion
    holds for thresholds $p_i\in(0,1)$ whenever
    $H(\cT,\cU^{(i)})/\delta_{p_i,K,C}\to0$, even if $p_i\to0$.
\end{rem}

\begin{prop}[Calibration and correctness of an ensemble]
\label{prop:ensemble_calibration}
Let $(X,\mu)$ be a data space with $0<\mu(X)<\infty$ and
$\cT:X\to Y$ the truth function, where $C=|Y|\geq2$.
Let $\cU$ consist of $K\geq1$ active models, each defined on all of $X$.
For model $l$, write $T_l\subseteq Y$ for its target set,
$t_l=|T_l|$, and $p_l(x)=f_l(\Phi_l(x))\in\Delta(T_l)$,
regarded as a probability vector in $\Delta(Y)$ by zero extension.
Put $\nu=\mu/\mu(X)$ and
\[
 \bar p(x)=\frac1K\sum_{l=1}^Kp_l(x),\qquad
 F(x)=\argmax_{a\in Y}\bar p_a(x),
\]
\[
 h=\frac1K\sum_{l=1}^K\int_X H_Y(p_l(x))\,d\nu(x).
\]
Here $H_Y$ uses logarithm base $C$; in terms of the original
per-model entropy, $H_Y(p_l)=\frac{\ln t_l}{\ln C}H(p_l)$
for $t_l\geq2$, and $H_Y(p_l)=0$ for $t_l=1$.
Assume each constituent model is calibrated on $X$:
\[
 \mathbb E_{\nu}\!\left[
  \mathbf 1_{\{\cT=a\}}\mid p_l\right]=(p_l)_a
  \quad (a\in Y,\ 1\leq l\leq K).
\]
Then the accuracy of the averaged ensemble satisfies
\begin{equation}\label{eq:ensemble_calibrated_accuracy}
 \operatorname{Acc}_X(F)
 :=\nu(\{F=\cT\})
 \geq 2C^{-h}-1
 \geq 1-2\ln(C)\,h.
\end{equation}
In particular, if $h\to0$ along a sequence of such calibrated
ensembles, their accuracies tend to $1$.
\end{prop}

\begin{proof}
At any $x$ with $F(x)\neq\cT(x)$, the averaged prediction
assigns at most $1/2$ to the true label. Thus
\[
 \mathbf 1_{\{F(x)\neq\cT(x)\}}
 \leq 2\bigl(1-\bar p_{\cT(x)}(x)\bigr).
\]
By calibration of each constituent model and the tower property,
\[
 \int_X\bar p_{\cT(x)}(x)\,d\nu(x)
 =\frac1K\sum_{l=1}^K\int_X
       \sum_{a\in Y}(p_l(x))_a^2\,d\nu(x).
\]
For every $q\in\Delta(Y)$, Jensen's inequality for the logarithm gives
\[
 \sum_{a\in Y}q_a^2
 \geq \exp\!\left(\sum_aq_a\ln q_a\right)
 =C^{-H_Y(q)}.
\]
Applying Jensen's inequality to the convex function $u\mapsto C^{-u}$
therefore yields
\[
 \int_X\bar p_{\cT(x)}(x)\,d\nu(x)
 \geq\frac1K\sum_{l=1}^K\int_X C^{-H_Y(p_l(x))}\,d\nu(x)
 \geq C^{-h}.
\]
Integrating the pointwise error bound proves the first inequality
in~\eqref{eq:ensemble_calibrated_accuracy}; the second follows from
$e^{-u}\geq1-u$ for $u\geq0$.
\end{proof}

\begin{rem}\label{rem:calibration_overlap}
Calibration is relative to the underlying data measure. In particular,
calibration of each model on its domain $\Dom(\Phi_l)$ does not imply
calibration after restricting to an overlapping region
$D_I=\bigcap_{l\in I}\Dom(\Phi_l)$: the distribution of true labels
conditional on a prediction may change under restriction.
Thus Proposition~\ref{prop:ensemble_calibration} cannot in general
be applied on $D_I$ merely from calibration on the individual domains.
Jensen--Shannon divergence, on the other hand, measures agreement of
predictions on the overlap without requiring calibration there.
Disagreement reveals differences between model predictions, although
agreement alone cannot guarantee correctness. Diversity of model training and architecture can make some errors visible through disagreement, while labeled validation also tests the remaining low-disagreement predictions. Proposition~\ref{prop:empirical_js_reliability}
provides a quantitative way to assess the correctness of JS-selected
cores using samples from the relevant data distribution.
\end{rem}

\begin{defn}[Empirical coverage and accuracy of a JS core]
\label{def:empirical_js_core}
Let $(X,\mu)$ be a finite measure space with $0<\mu(X)<\infty$,
let $\cT:X\to Y$ be the truth function, and put
$\nu=\mu/\mu(X)$. Fix an ensemble $\cU$, independently of the
sampled points below, whose models cover $X$ up to a null set.
Write $J_{\cU}:X\to[0,1]$ for its pointwise Jensen--Shannon
divergence and $F:X\to Y$ for its predicted label, using a fixed
tie-breaking rule. For $\tau\in[0,1]$, its \emph{JS core} is
\[
   \cC_\tau(\cU)=\{x\in X:J_{\cU}(x)\leq\tau\}.
\]
The \emph{population coverage} and \emph{accuracy} are
\[
   c_\tau=\nu(\cC_\tau(\cU)),\qquad
   \operatorname{Acc}(\cC_\tau)
   =1-\frac{\nu(\cC_\tau(\cU)\cap\{F\neq\cT\})}{c_\tau},
\]
where accuracy is defined when $c_\tau>0$.

Let $x_1,x_2,\ldots$ be independent, identically distributed points of
$X$ with distribution $\nu$. For a sample of size $N$, set
\[
\begin{aligned}
 N_\tau&=\#\{1\leq i\leq N:x_i\in\cC_\tau(\cU)\},\\
 M_\tau&=\#\{1\leq i\leq N:x_i\in\cC_\tau(\cU),\
                  F(x_i)\neq\cT(x_i)\}.
\end{aligned}
\]
The \emph{empirical coverage} and \emph{empirical accuracy} are
\[
   \hat c_\tau=\frac{N_\tau}{N},\qquad
   \widehat{\operatorname{Acc}}(\cC_\tau)
   =1-\frac{M_\tau}{N_\tau},
\]
where empirical accuracy is defined when $N_\tau>0$.
The counts and empirical quantities depend on $N$, which is suppressed
from their notation.
\end{defn}

\begin{prop}[Reliability of empirical JS cores]
\label{prop:empirical_js_reliability}
Assume the setting of Definition~\ref{def:empirical_js_core}. For
$\delta\in(0,1)$, put
\[
    \epsilon_N=\sqrt{\frac{\log(4/\delta)}{2N}}.
\]
With probability at least $1-\delta$, simultaneously for all
$\tau\in[0,1]$ with $\hat c_\tau>\epsilon_N$,
\begin{equation}\label{eq:js_validation_lower_bound}
   \operatorname{Acc}(\cC_\tau)
   \geq 1-\frac{M_\tau/N+\epsilon_N}{\hat c_\tau-\epsilon_N}.
\end{equation}
Moreover, for every $\kappa>\epsilon_N$, on the same event,
\begin{equation}\label{eq:js_validation_uniform}
  \sup_{\tau:\,\hat c_\tau\geq\kappa}
  \left|\operatorname{Acc}(\cC_\tau)
  -\widehat{\operatorname{Acc}}(\cC_\tau)\right|
  \leq\frac{2\epsilon_N}{\kappa}.
\end{equation}
In particular, for each fixed $\kappa>0$, as $N\to\infty$,
\[
  \sup_{\tau:\,\hat c_\tau\geq\kappa}
  \left|\operatorname{Acc}(\cC_\tau)
  -\widehat{\operatorname{Acc}}(\cC_\tau)\right|
  =O\!\left(\sqrt{\frac{\log N}{N}}\right)
  \quad\text{almost surely}.
\]
\end{prop}

\begin{proof}
Write
\[
 r_\tau=\nu(\cC_\tau(\cU)\cap\{F\neq\cT\}),\qquad
 \hat r_\tau=\frac{M_\tau}{N}.
\]
The Dvoretzky--Kiefer--Wolfowitz inequality applied to $J_{\cU}$ gives
\[
 \mathbb P\!\left(\sup_\tau|\hat c_\tau-c_\tau|>\epsilon\right)
 \leq 2e^{-2N\epsilon^2}.
\]
The same estimate holds for $r_\tau$ and $\hat r_\tau$. Indeed,
consider the measurable real-valued function
\[
 Z(x)=\begin{cases}
      J_{\cU}(x),&F(x)\neq\cT(x),\\
      2,&F(x)=\cT(x).
    \end{cases}
\]
For $\tau\in[0,1]$, the set $\{Z\leq\tau\}$ is exactly
$\cC_\tau(\cU)\cap\{F\neq\cT\}$, so the same inequality applies to
its empirical distribution function. Setting $\epsilon=\epsilon_N$
and using the union bound, with probability at least $1-\delta$,
\[
 |c_\tau-\hat c_\tau|\leq\epsilon_N,\qquad
 |r_\tau-\hat r_\tau|\leq\epsilon_N
 \quad\text{for every }\tau\in[0,1].
\]
Whenever $\hat c_\tau>\epsilon_N$, this implies
\[
 \operatorname{Acc}(\cC_\tau)
 =1-\frac{r_\tau}{c_\tau}
 \geq1-\frac{\hat r_\tau+\epsilon_N}{\hat c_\tau-\epsilon_N},
\]
proving~\eqref{eq:js_validation_lower_bound}.
For $\hat c_\tau\geq\kappa>\epsilon_N$, since $r_\tau\leq c_\tau$,
\[
\begin{aligned}
 \left|\frac{r_\tau}{c_\tau}-\frac{\hat r_\tau}{\hat c_\tau}\right|
 &\leq\frac{|r_\tau-\hat r_\tau|}{\hat c_\tau}
     +\frac{r_\tau}{c_\tau}
         \frac{|c_\tau-\hat c_\tau|}{\hat c_\tau}\\
 &\leq\frac{2\epsilon_N}{\kappa},
\end{aligned}
\]
which proves~\eqref{eq:js_validation_uniform}.
Finally, take $\delta_N=N^{-2}$ for $N\geq2$. Then
$\epsilon_N=\sqrt{\log(4N^2)/(2N)}
=O(\sqrt{\log N/N})$, and $\sum_{N\geq2}\delta_N<\infty$.
The Borel--Cantelli lemma gives the asserted almost-sure bound.
\end{proof}

\paragraph{Example (A JS core with 95\% empirical accuracy).}
Suppose an ensemble is trained on a dataset, and a JS threshold $\tau$
is selected from $N=10{,}000$ independent validation samples.
If $7{,}600$ samples lie in the core and $380$ of these are classified
incorrectly, then
\[
 N_\tau=7{,}600,\qquad M_\tau=380,\qquad
 \hat c_\tau=0.76,\qquad
 \widehat{\operatorname{Acc}}(\cC_\tau)=0.95.
\]
For $\delta=0.05$, we have
$\epsilon_N=\sqrt{\log(80)/20{,}000}\approx0.0148$.
Proposition~\ref{prop:empirical_js_reliability} therefore gives,
with confidence at least $95\%$,
\[
 \operatorname{Acc}(\cC_\tau)
 \geq1-\frac{0.038+0.0148}{0.76-0.0148}
 \approx 0.929.
\]
The guarantee also applies when the threshold $\tau$ was chosen using
these same validation samples, since the estimate is uniform in $\tau$.

\section{Entropy-based lifelong ensemble learning}

Lifelong learning addresses situations where a learner faces a stream of learning tasks, providing an opportunity for knowledge transfer across tasks \cite{Thrun1998, overcoming_forgetting, memory_based_lifelong, lifelong_review}.

In the context of the learning process in Definition \ref{def:learning_ensemble}, the sequence of tasks is to learn the truth function on expanding subsets of the domain over time. We consider a sequence of multiple learners where the earlier learners retain the memory of previous tasks, and new learners are created and transferred to new tasks. Earlier learners decide whether to pass a given task to the next generation by the entropy which measures uncertainty and disagreement among them. This multi-generational approach complements standard Deep Ensembles \cite{DeepEnsembles, Ovadia2019}, which can be fragile on outlier predictions. Concretely, this is determined by entropy thresholds and the core in Definition \ref{def:core}; the earlier learners are annihilated if their core becomes an empty set.

\begin{defn}[Lifelong learning]
	Let $E$ be a set whose elements are called environments, $A$ be a set whose elements are called actions, and
	$\cX_e = (\cX,\mu_e)$ where $\cX$ is a topological space called the data space with a probability measure $\mu_e$ that depends on environments $e\in E$. A learner is a tuple $(K,U,\pi)$ where $K$ is a topological space called the knowledge state space, $U: K \times \cX \to K$ is called the learning rule, and $\pi:K\times E \to A$ is called the performance policy.
	Let $P: A\times E\to\R$ that evaluates actions in different environments and $I:A \to \R$ that measures memory formation.
	
	A lifelong learning process is an infinite sequence $(e_i,x_i,k_i) \in E \times \cX \times K$, where $x_i$ are drawn according to the measure $\mu_{e_i}$ and $k_i$ is determined by $k_{i} = U(k_{i-1},x_{i-1})$ and the initial $k_0$ such that the evaluation and memory formation are non-decreasing up to a prefixed $\epsilon>0$:
	\begin{align*}
		P(\pi(k_{i+1},e_{i+1}),e_{i+1}) & > P(\pi(k_{i},e_{i+1}),e_{i+1}) - \epsilon; \\
		I(\pi(k_{i+1},e_{i+1})) &> I(\pi(k_{i},e_{i+1})) - \epsilon. 
	\end{align*}
\end{defn}

In a classification task, $\cX$ is the product $X\times Y$ of the input space and the labeling set. In practice, $\cX_e$ is replaced by its own product $\cX_e^N$ for some $N$, meaning that we take a batch of training and validation samples according to the measure $\mu_e$ rather than a single sample in each learning step. For a neural network model, $K$ is the parameter space of the model, $E$ is the set of all possible testing batches and $A$ is the set of all possible predictions (in the form of probability distributions on the labeling set) for the input testing batches.  
$P$ measures accuracy, whereas $I$ measures memory retention; $I$ need not equal Shannon predictive entropy, which tends to decrease under successful learning.

However, if the environment changes rapidly and drastically over time, catastrophic forgetting may happen \cite{catastrophic,van_de_Ven_2025} which crashes lifelong learning, namely the inequality for $I$ is violated. For this purpose, we consider an ensemble of learners instead to better retain memory. Retaining earlier specialists on their validated cores allows new tasks to be assigned to additional learners instead of requiring every previous model to be overwritten; this offers a concrete mechanism for reducing interference, although its effectiveness still depends on the routing and the changing data distribution. See also Nested Learning recently introduced in \cite{Google_nested}. 

\begin{defn}[Lifelong ensemble learning]
	\label{def:lifelong_ensemble}
	Let $E$, $A$, $\cX_e$, $P$ be the environment set, the action set, the data space and the evaluation of actions as before. Moreover, $I$ is given as a function $A^n \to \R$ that measures overall memory formation for an arbitrary number $n$ of learners.
	Given an infinite sequence of environments $e_i \in E$ which produce samples $x_i \in \cX_{e_i}$,
	a lifelong ensemble learning process is an infinite sequence 
	$$(\cU^{(i)},\Gamma^{(i)}, (k^{(i)}_l:l \in \cU^{(i)}))_{i=1}^\infty$$ 
	where each $\cU^{(i)}$ is an ensemble of finitely many learners, $\Gamma^{(i)}: A^{|\cU^{(i)}|} \to A$ is the aggregation of individual actions (voting), with injective functions $\cL^{(i)}: \cU^{(i)}\to \cU^{(i+1)}$ where $l\in \cU^{(i)}$ and $\cL^{(i)}(l)$ are the same learner (meaning $(K,U,\pi)$ remains the same), and individual knowledge states $k^{(i)}_l\in K^{(i)}_l$ are updated according to $U^{(i)}_l$ as before. The evaluation $P$ and memory formation $I$ for the aggregated action $\Gamma^{(i)}((\pi^{(i)}_l(k_l^{(i)},e_i): l \in \cU^{(i)}))$ are required to satisfy the same inequalities as before.
\end{defn}

The idea is to select entropy thresholds from labeled validation samples according to a prescribed accuracy criterion, and thereby determine the aggregation policy $\Gamma^{(i)}$. This suggests the following general procedure for a classification task.

\begin{enumerate}
	\item We start with a subset of the domain $S^{(1)}\subset X$. Using the data $\{(x,\cT(x)): x\in S^{(1)}\}$ (which is partitioned into training and validation subsets), we train and validate an ensemble of models $\cU^{(1)}$ whose domains are subsets of $S^{(1)}$ that cover $S^{(1)}$ and whose targets are subsets of $\cT(S^{(1)})$. The total entropy and Jensen--Shannon divergence defined above are measured with respect to $S^{(1)}$. (At this stage, $\mu(X-S^{(1)})$ is dropped since it is unknown and cannot be measured).
	\item For a preset high accuracy threshold, we use the validation data to find a low enough entropy threshold $p$ such that on the core $C_p(\cU^{(1)})$ of Definition \ref{def:core}, the ensemble $\cU^{(1)}$ reaches the required accuracy on the validation data. We restrict the active aggregation domains of the first-generation models to $C_p(\cU^{(1)})$, retaining their original prediction functions wherever they are needed to evaluate the routing score, and train the next generation of models $\cU^{(2)'}$ whose domain is restricted to be $S^{(1)} - C_p(\cU^{(1)})$. Then we take $\cU^{(2)} = \cU^{(1)} \cup \cU^{(2)'}$. Note that the entropy threshold $p$ could be zero, in which case the required accuracy threshold could not be reached and the first generation $\cU^{(1)}$ is annihilated.
	\item The above process can be repeated several times until the overall accuracy reaches a satisfactory level. If necessary, we may enlarge the models in the next generation to reach higher accuracy.
	\item Next, we enlarge the domain subset $S^{(1)}$ to $S^{(2)} \subset X$ by finding more samples from the real world or generating more samples by different means such as adding random noises, acting by symmetries, or adversarial attacks. The total entropy and Jensen--Shannon divergence are updated over the larger subset $S^{(2)}$ accordingly. The entropy thresholds in each layer of the ensembles are recalculated to match the required accuracy threshold. More generations of models are created as before and some earlier generations are annihilated, until the overall accuracy reaches a satisfactory level.
	\item The above procedure is repeated until the domain $X$ is saturated by training and validation samples.
\end{enumerate}

\section{Immunized systems under shifted and polluted environments}
\label{sec:experiments}

We investigate the practical effectiveness of entropy and JS disagreement
for core selection, shift detection, and routing on CIFAR-10
\cite{CIFAR-10} through the two-generation construction studied in this paper.
The general learning laws proved earlier do not depend on these experiments.
The first generation contributes to the routed prediction only on its
JS-based core, defined below. Its original prediction functions are
retained on the evaluation domain for computation of the routing score;
restricting the domain used for aggregation does not erase those functions.
Samples outside the first core provide the training domain for a second
generation, whose predictions are eligible only on the out-core of the
first generation.  The resulting routed ensemble, denoted by
$\IMM$, is a two-stage instance of the restricted-domain ensembles called
logifolds in~\cite{Jung-Lau2,Jung-Lau}.  This construction is related to
selective classification with an abstain option~\cite{CARL,SDIM,CPR}, but it
also uses rejected samples to extend the domain handled by the system.

The experiments follow this order.  We first test whether the proposed
Jensen--Shannon (JS) disagreement score behaves as intended when committee
members become more numerous and diverse.  We then freeze the first generation
and ask whether its core rejects previously unseen distribution shifts.  Only
after establishing these detection properties do we construct the second
generation from controlled polluted samples and measure how much routed
prediction enlarges the accurately covered domain.  Finally, we stress test
the resulting $\IMM$ under adversarial perturbations.

The unseen shifts comprise semantic out-of-distribution (OOD) inputs, in the
standard sense of test inputs drawn outside the training distribution
\cite{Hendrycks2017}, and label-preserving common corruptions
\cite{CIFAR10C}.  Neither shift is used to train the committee or tune its
threshold.  Adversarial perturbations are used separately as a controlled way
to generate polluted samples of adjustable difficulty.  The adversarial
experiments probe detection, specialization, and routing; they are not intended
to certify worst-case robustness.

\subsection{Experimental setup}
\label{subsec:exp_setup}

\paragraph{Models and data split.}
Each trial uses a stratified split of the 50{,}000 CIFAR-10 training images into
40{,}000 training and 10{,}000 validation images, together with the standard
10{,}000-image test set.  The baseline generation $\cU^{(0)}$ contains eight
classifiers: four VGG-type~\cite{VGG} and four ResNet-type~\cite{ResNet}
networks.  Its predictor is the probability average of the eight members.  We
add four ResNet specialists trained on weak untargeted pollution and four
specialists trained on weak targeted pollution, and define
\[
  \cU^{(1)}
  \coloneqq
  \cU^{(0)}\cup\mathrm{SP}_{\mathrm{u}}\cup\mathrm{SP}_{\mathrm{t}},
  \qquad K_1=16.
\]
We also write
$\cU^{(1)}_{\mathrm{u}}
\coloneqq\cU^{(0)}\cup\mathrm{SP}_{\mathrm{u}}$.

\paragraph{JS disagreement and cores.}
For model probabilities $p_1(x),\ldots,p_K(x)$ and their average
$\bar p(x)$, we use the equal-weight generalized Jensen--Shannon disagreement
score
\[
  J_{\cU}(x)
  =
  H\!\left(\bar p(x)\right)
  -\frac{1}{K}\sum_{k=1}^{K}H\!\left(p_k(x)\right),
  \qquad
  H(p)=-\sum_{a=1}^{10}p_a\log_{10}p_a.
\]
Thus, the logarithm base equals the number of target classes and
$0\leq J_{\cU}(x)\leq 1$.  The JS-based core
$\cC_{\tau}(\cU)$ and its empirical coverage and accuracy are defined
in Definition~\ref{def:empirical_js_core}.
These experiments were conducted before we formulated the more comprehensive
entropy core $C_p(\cU)$ in Definition~\ref{def:core}, which uses both the
entropy of the average prediction and the sum of constituent entropies.  We
expect similar results in the settings studied here, but have not repeated the
experiments using $C_p(\cU)$.
Core coverage is the fraction of all samples in $\cC_{\tau}(\cU)$, whereas
core accuracy is the conditional accuracy on this accepted subset.  For a
finite test environment $\cD$, the \emph{mean JS} reported below is precisely
\[
  \overline J_{\cU}(\cD)
  \coloneqq \frac{1}{|\cD|}\sum_{x\in\cD}J_{\cU}(x).
\]
It is an average committee-disagreement score, not the entropy of an empirical
distribution over inputs.

The diagnostic threshold in Table~\ref{tab:entropy_detection} is selected from
the clean validation set so that $\cU^{(0)}$ accepts approximately $90\%$ of
clean validation inputs.  In the two-generation experiments, $\tau_1$ is
selected separately for each controlled environment using its validation data,
with a target first-layer core accuracy of $95\%$; $\tau_2$ is selected after
the second generation is trained.  Both are then fixed on test data.  The shift
experiments use $\cU^{(1)}$ and the $\tau_1$ selected for the union environment,
but do not use the second generation.  No CIFAR-100, SVHN, or CIFAR-10-C sample
is used to fit a model or threshold.

Proposition~\ref{prop:empirical_js_reliability} gives a uniform
accuracy bound when the predictors are fixed independently of the samples
used for threshold selection. It therefore applies to threshold selection
for an independently trained, fixed first-layer ensemble. A later layer
whose training depends on validation-informed routing decisions need not
be independent of those same samples, so the proposition alone does not
certify the complete routed system from them. The reported test-set
accuracies are empirical evaluations, not confidence bounds derived from
that proposition.

\paragraph{Shifted test environments.}
For semantic OOD evaluation, CIFAR-100~\cite{CIFAR-10} is a near-domain shift
and the cropped Street View House Numbers (SVHN) test set~\cite{SVHN} is a
far-domain shift.  Their labels do not belong to the CIFAR-10 target space, so
we report first-layer acceptance and rejection rather than classification
accuracy.  For label-preserving covariate shift, we use the 15 standard
CIFAR-10-C corruptions at severities one through five~\cite{CIFAR10C}; core
accuracy remains meaningful because the CIFAR-10 labels are unchanged.

In addition to rejection at the fixed $\tau_1$, we report a threshold-free JS
area under the receiver operating characteristic curve (JS AUROC).  With clean
CIFAR-10 test input $X_{\mathrm{ID}}$ as the negative class and shifted input
$X_{\mathrm{S}}$ as the positive class, its population interpretation is
\[
 \operatorname{AUROC}_{\mathrm{JS}}
 =\Pr\!\left(J_{\cU^{(1)}}(X_{\mathrm{S}})
             >J_{\cU^{(1)}}(X_{\mathrm{ID}})\right)
 +\frac12\Pr\!\left(J_{\cU^{(1)}}(X_{\mathrm{S}})
             =J_{\cU^{(1)}}(X_{\mathrm{ID}})\right).
\]
The empirical value is computed from all clean--shifted score pairs.  Thus,
JS AUROC measures ranking over all possible score thresholds and must not be
confused with rejection at the fixed core threshold $\tau_1$.

\paragraph{Controlled polluted environments.}
We evaluate on clean data and on controlled polluted counterparts generated
from the same splits.  Let
\[
\begin{aligned}
  \cD_{\mathrm{weak}}
    &\coloneqq
    \cD_{\mathrm{weak,u}}\cup\cD_{\mathrm{weak,t}},\\
  \cD_{\mathrm{strong}}
    &\coloneqq
    \cD_{\mathrm{strong,u}}\cup\cD_{\mathrm{strong,t}},\\
  \cD_{\mathrm{union}}
    &\coloneqq
    \cD_{\mathrm{clean}}\cup
    \cD_{\mathrm{weak}}\cup
    \cD_{\mathrm{strong}},
\end{aligned}
\]
where $\mathrm{u}$ and $\mathrm{t}$ denote untargeted and targeted
perturbations.  All polluted samples are $\ell_2$-bounded and clipped to
$[0,1]$.  We use Auto-PGD with a cross-entropy objective (APGD-CE), an
iterative projected first-order attack with adaptive step-size control
\cite{auto_attack}.  Weak and strong samples are generated against the
probability average of $\cU^{(0)}$ using
\[
  (\epsilon,\epsilon_{\mathrm{step}},\text{steps},\text{restarts})
  \in\{(0.5,0.2,2,4),(0.7,0.2,10,4)\},
\]
respectively.  Weak samples are used for first-layer specialization, whereas
strong samples test behavior under more severe pollution.

\paragraph{Second-layer specialization and routing.}
Given a trained classifier and a polluted dataset, we replace the final
dense-plus-softmax head and fine-tune it using Adam~\cite{Adam} and categorical
cross-entropy for 21 epochs.  When sufficiently many first-layer out-core
samples are available, four out-core ResNet specialists form
$\cU^{(2)'}$.  For a test input, the $\IMM$ predictor is
\[
\Gamma^{(2)}(\IMM)(x)=
\begin{cases}
\Gamma^{(1)}(\cU^{(1)})(x),
  &J_{\cU^{(1)}}(x)\leq\tau_1,\\
\Gamma^{(2)'}(\cU^{(2)'})(x),
  &J_{\cU^{(1)}}(x)>\tau_1.
\end{cases}
\]
Here $J_{\cU^{(1)}}(x)$ is evaluated from the original, unrestricted
first-layer predictions even when $x$ lies outside the first core.
The first-layer restriction specifies where its prediction is used by the
aggregator, rather than the domain on which its stored gate score can be
computed. The final accepted core is the union of the first-layer core and those
first-layer rejections satisfying $J_{\cU^{(2)'}}(x)\leq\tau_2$.  The simple
probability average over $\cU^{(1)}\cup\cU^{(2)'}$, without this gate, is
denoted by All-ens.  Unless stated otherwise, results are reported as the mean
and sample standard deviation over five independently trained trials.

\subsection{First-layer disagreement and detection}
\label{subsec:first_layer_detection}

\paragraph{Detecting controlled pollution.}
Table~\ref{tab:entropy_detection} compares the global accuracy of the
family-restricted ResNet and VGG committees with global accuracy, core
coverage, and core accuracy of successively larger committees.  Adding
weak-pollution specialists both increases membership and
introduces models with different training experience.  On weak untargeted
pollution, core accuracy rises from $86.73\%$ for $\cU^{(0)}$ to $94.66\%$ for
$\cU^{(1)}_{\mathrm u}$ and $95.39\%$ for $\cU^{(1)}$.  Under strong untargeted
pollution, coverage contracts from $65.96\%$ to $22.38\%$ while core accuracy
rises from $18.26\%$ to $49.76\%$.  The same threshold therefore rejects more
inputs as disagreement among the enlarged, more diverse committee increases.
This behavior agrees with the established connection between ensemble
diversity and informative uncertainty~\cite{DeepEnsembles,Pang2019}, and
verifies that the proposed JS score acts in the intended direction.  Since
committee size and specialization change together, this is a diagnostic of the
complete first-layer design rather than a causal ablation of either factor.

\begin{table}[!htbp]
\centering
\caption{Detection results over five independent trials.  An entry $a(b)$ is
the mean $a$ and sample standard deviation $b$, in percent.  ResNet and VGG
rows report global accuracy only.  In each ensemble cell, the first line is
global accuracy, the second line is core accuracy, and the third line reports
the corresponding core coverage as ``with coverage.''
Within each trial, one diagnostic threshold is selected from clean validation
data for $90\%$ coverage by $\cU^{(0)}$ and applied to every displayed
environment and ensemble; its cross-trial value is $0.1830(0.0038)$.}
\label{tab:entropy_detection}
\setlength{\tabcolsep}{2.5pt}
\scriptsize
\resizebox{\linewidth}{!}{%
\begin{tabular}{lccccc}
\toprule
System & Clean & Weak untargeted & Weak targeted &
Strong untargeted & Strong targeted\\
\midrule
ResNets (global acc.) &
$93.53(0.11)$ & $66.64(0.35)$ & $68.64(0.30)$ &
$13.03(0.81)$ & $17.61(1.00)$\\
VGGs (global acc.) &
$93.29(0.12)$ & $75.73(0.31)$ & $77.86(0.23)$ &
$18.08(0.66)$ & $24.27(0.81)$\\
\midrule
$\cU^{(0)}$ \makecell{global acc. \\ core acc.\\ with core coverage}&
\makecell{$94.13(0.18)$\\$97.51(0.11)$\\with $89.69(0.18)$ } &
\makecell{$65.87(0.28)$\\$86.73(0.25)$\\with $73.61(0.31)$ } &
\makecell{$67.65(0.28)$\\$82.71(0.50)$\\with $80.18(0.55)$ } &
\makecell{$13.19(0.80)$\\$18.26(1.09)$\\with $65.96(0.79)$ } &
\makecell{$17.66(1.03)$\\$19.36(1.03)$\\with $90.84(0.51)$ }\\
\cmidrule(lr){1-6}
$\cU^{(1)}_{\mathrm{u}}$ \makecell{global acc. \\ core acc.\\ with core coverage}&
\makecell{$94.08(0.12)$\\$97.69(0.10)$\\with $87.95(0.49)$ } &
\makecell{$77.38(0.27)$\\$94.66(0.49)$\\with $66.32(0.56)$ } &
\makecell{$78.23(0.39)$\\$92.81(0.67)$\\with $69.82(0.73)$ } &
\makecell{$21.14(0.57)$\\$43.58(2.80)$\\with $25.84(1.15)$ } &
\makecell{$23.66(0.85)$\\$41.33(2.44)$\\with $40.53(1.65)$ }\\
\cmidrule(lr){1-6}
$\cU^{(1)}$ \makecell{global acc. \\ core acc.\\ with core coverage}&
\makecell{$93.98(0.07)$\\$97.65(0.13)$\\with $87.64(0.45)$ } &
\makecell{$82.37(0.52)$\\$95.39(0.53)$\\with $65.71(0.66)$ } &
\makecell{$82.85(0.48)$\\$94.18(0.51)$\\with $68.65(0.72)$ } &
\makecell{$39.88(0.41)$\\$49.76(3.10)$\\with $22.38(1.06)$ } &
\makecell{$36.93(0.64)$\\$47.45(2.87)$\\with $34.59(1.48)$ }\\
\bottomrule
\end{tabular}}
\end{table}

\subsection{First-layer rejection under unseen shifts}
\label{subsec:ood_results}

\paragraph{Semantic OOD rejection.}
We apply the first layer $\cU^{(1)}$, with its threshold selected on union validation data, to CIFAR-100 and SVHN
without retraining or threshold selection.  It accepts $x$ precisely when
$J_{\cU^{(1)}}(x)\leq\tau_1$.  Every percentage in
Table~\ref{tab:semantic_ood} uses the entire OOD test set as its denominator;
for example, an $86\%$ rejection rate means that $86\%$ of all OOD inputs are
declared out of core.  Core accuracy is undefined because neither dataset has
CIFAR-10-compatible targets.

\begin{table}[htbp]
\centering
\caption{First-layer semantic OOD detection over five independent trials.
An entry $a(b)$ denotes mean(sample standard deviation), in percent.  The
$\cU^{(1)}$ threshold was selected on union validation data and frozen before
observing OOD inputs.  JS AUROC uses clean CIFAR-10 test inputs as negatives
and the indicated OOD inputs as positives.}
\label{tab:semantic_ood}
\setlength{\tabcolsep}{8pt}
\small
\begin{tabular}{lccc}
\toprule
OOD dataset & JS AUROC & Core coverage (acceptance) & Rejection\\
\midrule
CIFAR-100 & $88.69(0.09)$ & $13.72(1.08)$ & $86.28(1.08)$\\
SVHN      & $88.29(0.32)$ &  $9.75(1.26)$ & $90.25(1.26)$\\
\bottomrule
\end{tabular}
\end{table}

The first layer rejects most examples from both the near-domain CIFAR-100
shift and the visually more distant SVHN shift.  The corresponding JS AUROCs
are close to $89\%$, so the result is not solely an artifact of the chosen
$\tau_1$: shifted samples tend to receive larger JS scores throughout the
ranking.  These results establish an abstention capability of the fixed first
layer; they do not construct an OOD-adapted model or assign CIFAR-10 accuracy
to semantically incompatible labels.

\paragraph{Label-preserving common corruptions.}
CIFAR-10-C measures a different shift: the target class remains valid, so one
can compare global accuracy with accuracy conditional on first-layer
acceptance.  Table~\ref{tab:cifar10c} averages equally over the 15 corruption
types at each severity.  As severity increases from one to five, global
accuracy falls from $90.74\%$ to $69.70\%$ and core coverage contracts from
$64.89\%$ to $32.06\%$, yet core accuracy remains $93.64\%$ even at severity
five.  At the same time, rejection rises from $35.11\%$ to $67.94\%$ and JS
AUROC rises from $59.10\%$ to $79.68\%$.  Thus mild corruptions are difficult
to rank, but increasing corruption produces stronger disagreement and more
abstention while retaining a comparatively reliable core.

\begin{table}[htbp]
\centering
\caption{First-layer results on CIFAR-10-C over five independent trials.  For
each trial and severity, metrics are averaged equally over the 15 corruption
types; $a(b)$ denotes the cross-trial mean(sample standard deviation), in
percent.  Core coverage is acceptance at the frozen union $\tau_1$, rejection
is its complement, and core accuracy is conditional on acceptance.  JS AUROC
compares clean and corrupted CIFAR-10 test inputs.}
\label{tab:cifar10c}
\setlength{\tabcolsep}{4pt}
\small
\resizebox{\linewidth}{!}{%
\begin{tabular}{cccccc}
\toprule
Severity & Global accuracy & Core accuracy & Core coverage & Rejection & JS AUROC\\
\midrule
1   & $90.74(0.09)$ & $98.97(0.11)$ & $64.89(1.11)$ & $35.11(1.11)$ & $59.10(0.13)$\\
2   & $87.95(0.18)$ & $98.54(0.18)$ & $57.41(1.16)$ & $42.59(1.16)$ & $64.79(0.13)$\\
3   & $84.23(0.26)$ & $97.66(0.32)$ & $50.12(1.13)$ & $49.88(1.13)$ & $69.51(0.16)$\\
4   & $78.94(0.32)$ & $96.52(0.47)$ & $42.09(1.05)$ & $57.91(1.05)$ & $74.24(0.17)$\\
5   & $69.70(0.45)$ & $93.64(0.81)$ & $32.06(0.96)$ & $67.94(0.96)$ & $79.68(0.23)$\\
\midrule
All & $82.31(0.26)$ & $97.07(0.35)$ & $49.31(1.07)$ & $50.69(1.07)$ & $69.46(0.14)$\\
\bottomrule
\end{tabular}}
\end{table}

\subsection{Second-generation specialization and routing}
\label{subsec:second_generation}

\paragraph{Construction across controlled environments.}
We form the second generation from the first-layer out-core samples of each
controlled environment.  Table~\ref{tab:second_generation} compares the
baseline, specialized first layer, simple average of all available members,
and routed $\IMM$.  On the union environment, routing raises global accuracy
from $67.20\%$ for $\cU^{(1)}$ and $78.44\%$ for All-ens to $86.83\%$.  Its
final core covers $76.00\%$ of the union test set at $94.63\%$ core accuracy,
compared with $46.77\%$ coverage at $94.84\%$ core accuracy for the first
layer.  Under weak pollution, the corresponding coverage increase is smaller,
from $66.83\%$ to $71.12\%$, while core accuracy remains near $95\%$.

These comparisons give a useful interpretation of routing.  Relative to the
roughly $94\%$ clean global accuracy, global performance declines in the weak
and union environments, but the accepted portions of both $\cU^{(1)}$ and
$\IMM$ maintain approximately $95\%$ accuracy.  The second layer primarily
enlarges this accurately covered portion.  In the weak environment, All-ens
and $\IMM$ have similar global accuracies, $85.02\%$ and $85.83\%$.
Weak samples already resemble the pollution used to train the first-layer
specialists, leaving less residual structure for routing to exploit; simple
averaging therefore captures much of the benefit.  The broader union
environment contains a larger and more heterogeneous out-core, for which a
dedicated second layer has a much larger advantage.

The weak environment gives $\tau_1=0.1797(0.0053)$ and
$\tau_2=0.0231(0.0141)$, whereas the union gives
$\tau_1=0.1008(0.0060)$ and $\tau_2=0.0243(0.0050)$; parentheses again contain
the sample standard deviation.  These values are independently selected on
validation data for the JS definition above, not inherited constants.  For
$\IMM$, mean JS is the route-selected average
\[
 \overline J_{\IMM}(\cD)
 =\frac{1}{|\cD|}\left(
   \sum_{x\in\cD\cap\cC_{\tau_1}(\cU^{(1)})}J_{\cU^{(1)}}(x)
   +\sum_{x\in\cD\setminus\cC_{\tau_1}(\cU^{(1)})}J_{\cU^{(2)'}}(x)
 \right).
\]
It is a disagreement diagnostic and is not identified with the non-decreasing
memory functional in Definition~\ref{def:lifelong_ensemble}.

\begin{table}[htbp]
\centering
\caption{Second-generation specialization and routing over five independent
trials.  Each entry is mean(sample standard deviation).  Accuracy and core
quantities are percentages; mean JS is on its base-10 scale.  For
$\cU^{(1)}$, core columns refer to the first-layer core; for $\IMM$, they refer
to the routed final core.}
\label{tab:second_generation}
\setlength{\tabcolsep}{3pt}
\scriptsize
\resizebox{\linewidth}{!}{%
\begin{tabular}{llcccc}
\toprule
Environment & System & Accuracy (\%) & Mean JS &
Core coverage (\%) & Core accuracy (\%)\\
\midrule
Clean & $\cU^{(0)}$ & $94.13(0.18)$ & $0.0466(0.0003)$ & -- & --\\
Clean & $\cU^{(1)}$ & $93.98(0.07)$ & $0.0585(0.0004)$ & -- & --\\
\midrule
Weak & $\cU^{(0)}$ & $66.76(0.28)$ & $0.0834(0.0008)$ & -- & --\\
Weak & $\cU^{(1)}$ & $82.61(0.50)$ & $0.1057(0.0012)$ &
$66.83(0.43)$ & $95.05(0.31)$\\
Weak & All-ens & $85.02(0.39)$ & $0.1724(0.0150)$ & -- & --\\
Weak & $\IMM$ & $85.83(0.76)$ & $0.0680(0.0016)$ &
$71.12(2.76)$ & $94.99(0.32)$\\
\midrule
Union & $\cU^{(0)}$ & $51.70(0.34)$ & $0.0812(0.0009)$ & -- & --\\
Union & $\cU^{(1)}$ & $67.20(0.35)$ & $0.1465(0.0010)$ &
$46.77(0.54)$ & $94.84(0.25)$\\
Union & All-ens & $78.44(0.54)$ & $0.2101(0.0022)$ & -- & --\\
Union & $\IMM$ & $86.83(0.32)$ & $0.0469(0.0008)$ &
$76.00(0.70)$ & $94.63(0.19)$\\
\bottomrule
\end{tabular}}
\end{table}

\subsection{Adversarial stress tests}
\label{subsec:adversarial_results}

\paragraph{Direct, diverted, and transfer attacks.}
Our intended threat model is a \emph{query-only black-box attack} on the
complete predictor $F_{\IMM}=\Gamma^{(2)}(\IMM)$: the attacker can query
$F_{\IMM}(z)$ and observe its prediction probabilities, but does not see
the constituent classifiers, their parameters, or the routing implementation.
Write $r(z)\in\{1,2\}$ for the route selected at $z$ and $F_j$ for the
prediction function of branch $j$, so that $F_{\IMM}(z)=F_{r(z)}(z)$.
Away from the switching set $J_{\cU^{(1)}}(z)=\tau_1$, where the
routing is locally constant, the derivative of the composite predictor is
exactly that of the selected branch. A query-only attacker could in principle
estimate this local derivative from function queries, given probability
outputs and sufficient queries.

For the reported diverted APGD experiment, we use a \emph{fixed-initial-route
gradient approximation} to this threat model.  For each original input $x_0$,
we use the implemented system to determine $r(x_0)$, and throughout the
multistep optimization we differentiate the fixed surrogate
\[
    \widetilde F_{x_0}(z):=F_{r(x_0)}(z)
    \qquad\text{in place of}\qquad F_{\IMM}(z)=F_{r(z)}(z).
\]
This substitutes exact internal branch gradients for the gradients that a
query-only attacker would have to estimate.  As long as the iterate $z$ stays
in a region with the same active route $r(z)=r(x_0)$, the gradients agree;
after it crosses the routing boundary, they need not.  The approximation
is therefore most faithful over short attack trajectories with few, sufficiently
small steps that do not change the route.  With many optimization steps,
the iterates have more opportunities to cross routing boundaries, and the
fixed-route attack may depart substantially from an attack that queries
the full routed function at each step.  This depends on the attack trajectory
and step sizes, not on the number of steps alone.  Every final perturbed
input is nevertheless evaluated using the actual composite predictor and
its gate. This evaluation is therefore a \emph{gradient-oracle proxy} for a
query-only black-box attack, not an actual query-only implementation, and
not a fully route-adaptive attack. The use of exact branch gradients can
strengthen the attack locally, while freezing the initial route can weaken
it after route changes, so neither protocol is an unconditional bound on the
other.

The $\IMM$ used here is formed on $\cD_{\mathrm{union}}$.  For
$\cU^{(0)}$, $\cU^{(1)}$, and All-ens, APGD-CE is generated directly against
the corresponding probability-average predictor.  In particular, All-ens is
treated as one classifier: its untargeted sample is generated against, and
then evaluated by, the same all-member average.  For $\IMM$, APGD uses
$\Gamma^{(1)}(\cU^{(1)})$ throughout optimization for a clean first-core
input and $\Gamma^{(2)'}(\cU^{(2)'})$ otherwise; the perturbed input is then
evaluated by the complete gate.  As a transfer test, AutoAttack samples
generated on a RobustBench surrogate are evaluated by each fixed predictor
\cite{auto_attack,robustbench}.

Table~\ref{tab:direct_attacks} reports global accuracy together with coverage
and conditional accuracy of the surviving core.  The $67.31\%$ core accuracy
of All-ens under untargeted APGD must be read with its $25.05\%$ coverage: the
gate retains a small, easier subset even though global accuracy is only
$23.95\%$.  Inputs on which all members already agree also tend to be accepted
by the first layer of $\IMM$; routing mainly changes the uncertain remainder.
The second layer recovers many of these inputs, giving $67.78\%$ coverage,
$63.31\%$ core accuracy, and $63.02\%$ global accuracy.  This explains why
routing changes global accuracy much more than the conditional accuracy of the
small All-ens core.  On transferred AutoAttack samples, by contrast, All-ens
has higher global accuracy ($89.31\%$ versus $86.06\%$), whereas $\IMM$
accepts a much larger core ($61.99\%$ versus $35.03\%$).

The diverted $\IMM$ attack is a fixed-initial-route gradient-oracle
approximation to the query-only black-box setting, not a literal query-only
attack with a measured query budget. It does not certify worst-case robustness
of the complete predictor. Our two-layer system is evaluated against these
approximate attacks. We expect that extending the hierarchy with additional
layers trained on adversarial samples may help defend against stronger
multi-step attacks, especially when attack trajectories cross routing
boundaries. This possibility has not been tested. We leave evaluation of
deeper hierarchies under full query-only attacks---using repeated queries
to the complete routed predictor and allowing route changes---to future work.
A route-updating or boundary-aware white-box adaptive attack (for example,
using a smooth gate or a BPDA-style approximation) would probe
complementary weaknesses such as gradient masking
\cite{Athalye2018,Uesato2018}.

\begin{table}[htbp]
\centering
\caption{Attack results over five independent trials.  In each cell, the first
line is global accuracy, the second line is core accuracy, and the third line
reports the corresponding core coverage as ``with coverage'';
$a(b)$ denotes mean(sample standard deviation), in percent.  For non-gated
predictors, the core is computed with the union $\tau_1$; for $\IMM$, it is the
final core under $\tau_1$ and $\tau_2$.  APGD is direct for non-gated predictors
and diverted for $\IMM$; the black-box column uses transferred AutoAttack
samples.  Bold marks, separately within each attack column, the highest global
accuracy and the highest core accuracy; coverage is not ranked.}
\label{tab:direct_attacks}
\setlength{\tabcolsep}{3pt}
\scriptsize
\resizebox{\linewidth}{!}{%
\begin{tabular}{lccc}
\toprule
Predictor & Strong untargeted APGD & Strong targeted APGD & Black-box transfer\\
\midrule
$\cU^{(0)}$ \makecell{global acc. \\ core acc. \\ with core coverage}&
\makecell{$13.19(0.80)$\\$23.66(1.48)$\\with $47.92(1.04)$} &
\makecell{$17.66(1.03)$\\$22.58(1.18)$\\with $77.63(1.05)$} &
\makecell{$77.76(0.38)$\\$92.31(0.26)$\\with $35.29(1.72)$}\\
\cmidrule(lr){1-4}
$\cU^{(1)}$  \makecell{global acc. \\ core acc. \\ with core coverage} &
\makecell{$16.33(0.67)$\\$36.16(3.92)$\\with $39.47(3.32)$} &
\makecell{$17.98(0.80)$\\$25.63(2.32)$\\with $69.95(4.35)$} &
\makecell{$87.26(0.26)$\\$97.73(0.10)$\\with $39.90(2.71)$}\\
\cmidrule(lr){1-4}
All-ens  \makecell{global acc. \\ core acc. \\ with core coverage} &
\makecell{$23.95(1.14)$\\$\mathbf{67.31(1.19)}$\\with $25.05(0.66)$} &
\makecell{$29.17(1.42)$\\$40.79(0.85)$\\with $57.51(0.96)$} &
\makecell{$\mathbf{89.31(0.16)}$\\$\mathbf{99.56(0.12)}$\\with $35.03(2.17)$}\\
\cmidrule(lr){1-4}
$\IMM$  \makecell{global acc. \\ core acc. \\ with core} &
\makecell{$\mathbf{63.02(2.35)}$\\$63.31(3.63)$\\with $67.78(1.05)$} &
\makecell{$\mathbf{48.76(4.30)}$\\$\mathbf{43.99(4.20)}$\\with $88.54(1.41)$} &
\makecell{$86.06(0.76)$\\$97.39(0.15)$\\with $61.99(1.68)$}\\
\bottomrule
\end{tabular}}
\end{table}

\paragraph{Summary of the empirical evidence.}
The controlled-pollution diagnostic shows that enlarging and diversifying the
committee makes JS-based core selection more discriminating.  The frozen first
layer also rejects most semantic OOD inputs and abstains increasingly often as
CIFAR-10-C severity grows, without observing either shift during fitting.
Second-generation specialization then enlarges coverage while preserving core
accuracy near $95\%$, with its clearest benefit on the heterogeneous union
environment.  Finally, the diverted-attack results show that routing can
recover inputs outside the first core, but they do not constitute an adaptive
robustness certificate.  Within the scope of two generations on CIFAR-10, the
five trials therefore support JS disagreement as a reproducible signal for
core selection, shift detection, and controlled expansion of the accurately
covered domain.

\subsection{Conclusion}
We established two laws relating work and entropy in the learning dynamics of
models and ensembles, and formulated theoretical entropy cores that require
low predictive uncertainty and low committee disagreement. In the CIFAR-10
experiments, we use JS cores based on disagreement alone, and the pointwise JS score
behaves consistently with this role: enlarging and diversifying the first
generation makes its core more selective under controlled pollution while
retaining a more accurate subset.

The same first-layer gate also transfers beyond the environment used to fit
it.  With its threshold frozen, it rejects most CIFAR-100 and SVHN inputs and
abstains increasingly often as CIFAR-10-C corruption severity grows.  On the
label-preserving corruptions, the accepted subset remains substantially more
accurate than the full input stream.  These observations support JS
disagreement as a practical signal for both semantic OOD rejection and
selective prediction under covariate shift.

For controlled polluted environments, specializing a second generation on the
first out-core enlarges coverage while maintaining core accuracy near $95\%$.
The routed system $\IMM$ has its clearest advantage on the heterogeneous union
environment and also improves performance in the diverted APGD stress tests,
which use an idealized fixed-initial-route gradient approximation to a
query-only black-box attack on the complete prediction function. Internal
branch gradients are used in place of query-based gradient estimates, and
the route is not updated during optimization. The present evidence is limited
to two generations on CIFAR-10, and the diverted attack is neither a literal
query-only evaluation nor a worst-case robustness certificate. We expect that
additional layers trained on adversarial samples may offer improved defenses
against stronger multi-step attacks, but testing this expectation with
full query-only evaluations is left to future work.
Within that scope,
the results suggest a principled way to expand modular ensembles: identify a
reliable core by low JS disagreement, specialize on its complement, and route
predictions over the resulting restricted domains.

\section*{Impact Statement}

This paper presents work whose goal is to advance the field of Machine
Learning. There are many potential societal consequences of our work, none
which we feel must be specifically highlighted here.


\bibliographystyle{plain}
\bibliography{bib}
\newpage
\appendix
\onecolumn

\clearpage

\end{document}